\documentclass[10pt]{article}
\usepackage[letterpaper,margin=0.75in]{geometry}
\usepackage[T1]{fontenc}
\usepackage{lmodern}
\usepackage{amsmath,amssymb,amsthm,mathtools,bm}
\usepackage{microtype,graphicx,booktabs,array,enumitem,caption}
\usepackage{xcolor}
\definecolor{accent}{HTML}{000000}
\usepackage[round]{natbib}
\usepackage[hidelinks]{hyperref}
\hypersetup{pdftitle={Cheap to Draw, Expensive to Trust: Certifying Test-Time Scaling Curves},
  pdfauthor={Sohail (Neel) Sarkar and Shakuntala Baichoo}}
\setlist{itemsep=2pt,topsep=4pt}

\makeatletter
\renewcommand\section{\@startsection{section}{1}{\z@}%
  {-3.2ex \@plus -1ex \@minus -.2ex}{1.4ex \@plus .2ex}%
  {\normalfont\large\bfseries}}
\renewcommand\subsection{\@startsection{subsection}{2}{\z@}%
  {-2.6ex\@plus -1ex \@minus -.2ex}{1ex \@plus .2ex}%
  {\normalfont\normalsize\bfseries}}
\renewcommand\paragraph{\@startsection{paragraph}{4}{\z@}%
  {1.6ex \@plus .5ex \@minus .2ex}{-1em}%
  {\normalfont\normalsize\bfseries}}
\newcommand\floatbarrier{\par\begingroup\let\@elt\relax\edef\@tempa{\@deferlist\@dbldeferlist}\ifx\@tempa\@empty\endgroup\else\endgroup\clearpage\fi}
\makeatother

\newtheoremstyle{accent}{\topsep}{\topsep}{\itshape}{}{\bfseries\color{accent}}{.}{.5em}{}
\theoremstyle{accent}
\newtheorem{theorem}{Theorem}
\newtheorem{proposition}[theorem]{Proposition}
\newtheorem{lemma}[theorem]{Lemma}
\newtheorem{corollary}[theorem]{Corollary}
\theoremstyle{remark}

\newcommand{\eps}{\varepsilon}
\newcommand{\E}{\mathbb E}
\newcommand{\Prb}{\mathbb P}
\newcommand{\Ber}{\operatorname{Bernoulli}}
\newcommand{\Bin}{\operatorname{Bin}}
\newcommand{\kl}{\operatorname{kl}}
\newcommand{\vbar}{\bar v}
\newcommand{\sumw}{\Sigma_{\mathrm w}}
\newcommand{\ind}{\mathbf 1}
\DeclareMathOperator{\KL}{KL}
\DeclareMathOperator{\Var}{Var}

\title{\color{accent}\bfseries Cheap to Draw, Expensive to Trust:\\ Certifying Test-Time Scaling Curves}
\author{Sohail (Neel) Sarkar \qquad Shakuntala Baichoo\\[.4ex]
\normalsize PMCC AI Lab, Peter Munk Cardiac Centre\\
\normalsize University Health Network, Toronto, Ontario, Canada\\[.3ex]
\normalsize\href{mailto:sohail.sarkar@uhn.ca}{\texttt{sohail.sarkar@uhn.ca}} \qquad
\href{mailto:shakuntala.baichoo@uhn.ca}{\texttt{shakuntala.baichoo@uhn.ca}}}
\date{}

\begin{document}
\twocolumn[{%
\begin{@twocolumnfalse}
\maketitle

\begin{center}
\begin{minipage}{0.88\textwidth}
\begin{abstract}
\normalsize
Sampling several answers and keeping the one a verifier scores highest is one of the simplest ways to buy accuracy at test time. Its effect is reported as a scaling curve: accuracy against the number $k$ of sampled answers. The curve is cheap to draw and expensive to trust. A budget read off it is chosen after looking at every point, so only a band that covers all budgets at once protects the choice, and on a 100-question benchmark a fixed exact-binomial design needs 192,000 generated answers to certify 64 budgets to within $\pm1/32$ at 95\%. Most of that cost pays for the wrong uncertainty. A benchmark is a fixed list of questions; at budget 64, about three quarters of the variance of a selected answer's correctness lies between questions, and an audit that revisits every question need not pay for it. We derive the minimax cost of certifying the whole curve, up to logarithmic factors. It has three parts: calibrating the tail of the score distribution, telling the questions apart, and within-question noise summed along the curve. At a single benchmark the last part sharpens to the variance of one answer's influence under the best allocation of answers to questions, which every valid audit pays and an audit that learns the allocation attains, up to a logarithm, as the precision grows. A paired audit built on an exponential inequality for two independent draws at the same question needs no pilot. On 185 held-out score pools it uses 0.74 times the answers of the cheapest competing certified audit at 64 budgets and 0.53 times at 1,024, and on a newly generated MMLU-Pro study it certified the curve with 79,133 answers, within 0.6\% of what a cost law fitted beforehand predicted from the study's within-question variance. The same paths certify pass@$k$ and majority voting, and the bands extend to populations of questions and to answers that depend on earlier ones.
\end{abstract}

\end{minipage}
\end{center}
\end{@twocolumnfalse}
}]

\section{Introduction}\label{sec:intro}

Sampling several answers from a language model and returning the one a verifier scores highest is among the simplest ways to buy accuracy with computation at test time \citep{codex,webgpt,largelanguagemonkeys,weaver2025}. Its effect is reported as a scaling curve: the accuracy of the selected answer at each budget $k=1,\dots,K$ on a fixed benchmark (Figure~\ref{fig:hero}a). Readers use the curve to decide how many samples are worth paying for, and they usually decide on a plateau. On the 185 held-out score pools studied below, the median curve comes within $1/32$ of its best value by $k=10$. The differences that matter are a few points.

A budget read off a curve is chosen after looking at every point. Only a band that holds at all budgets at once protects that choice. Pointwise intervals do not, and neither does the bootstrap: on the same pools, a within-question bootstrap band at nominal level 95\% missed the exact curve in 58 of 925 runs. Valid bands exist, and they are expensive. For 100 MMLU-Pro questions, one path of 64 answers per question, 6,400 answers in all, already gives an unbiased estimate of every point. Certifying the curve to within $\pm1/32$ with a fixed exact-binomial design takes 192,000 (Figure~\ref{fig:hero}c).

Most of that cost pays for the wrong uncertainty. The benchmark is a fixed list, and the accuracy it reports is an average over that list. An audit that samples questions at random pays for which questions it happened to draw. At budget 64, about three quarters of the variance of the selected answer's correctness is of that kind (Figure~\ref{fig:hero}b). Visiting every question in every round removes it, as stratified sampling removes between-stratum variance \citep{cochran1977}. Balance by itself buys nothing, though. An exact-binomial audit costs the same on balanced rounds as on random ones (Section~\ref{sec:ablation}), because its interval cannot see a smaller variance. The saving appears only when the confidence sequence can see it, and its size depends on how the remaining variance accumulates along a whole curve. Working this out is the subject of the paper.

\begin{figure*}[t]
\centering\includegraphics[width=\textwidth]{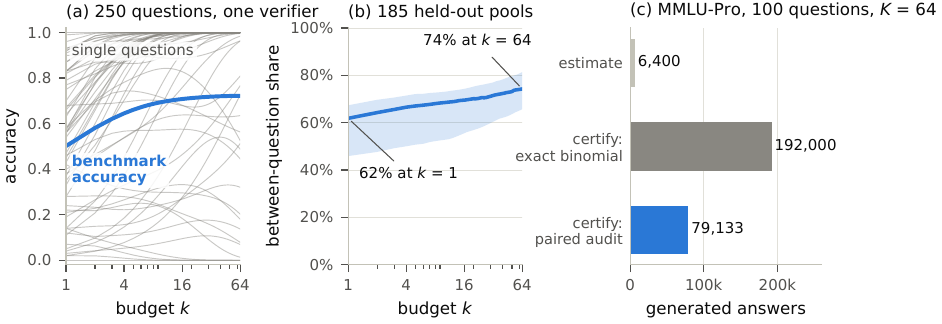}
\caption{Why a certified curve can be cheap. (a) Accuracy curves $p_k(x)$ of single questions spread from 0 to 1, while the benchmark average rises smoothly (MATH500, answers from Llama-3.1-8B-Instruct, one reward model; 60 of the 250 questions are drawn). (b) Median share of the variance of the selected answer's correctness that lies between questions, with interquartile band, over the 185 held-out pools. (c) Generated answers needed for 100 MMLU-Pro questions and 64 budgets: an unbiased estimate of every point, a band from a fixed exact-binomial design, and a band from the paired audit of Section~\ref{sec:audit}, both at half-width $1/32$ with 95\% simultaneous coverage.}
\label{fig:hero}
\end{figure*}

\paragraph{What a certified curve costs.} We model an audit that may choose questions, grow paths of answers, query correctness and stop, all adaptively, and must return a band of half-width $\eps$ that covers every budget with probability $1-\delta$, whatever the answers look like. Three resources are counted: generated answers, correctness queries and question visits. The results are these.
\begin{itemize}[leftmargin=*]
\item \emph{The price on a fixed list} (Theorem~\ref{thm:frontier}). On a benchmark of $M$ questions, certifying all budgets up to $K$ takes $K/\eps+\min(M,\eps^{-2})+s/\eps^2$ generated answers up to logarithmic factors, where $s$ bounds the within-question variance summed along the curve. The three terms are three obstructions: a rare high-scoring answer that decides the largest budgets, the identity of the questions, and within-question noise. Each has a matching lower bound, and the first is paid at every law. Correctness queries obey a law of the same shape, and labels for the whole curve cost what labels for its first point cost, up to logarithms. A coupling of nested winners on one path (Lemma~\ref{lem:coupling}) is what lets one path serve every budget. For fresh questions from a population the price is $K/\eps^2$ instead, while simultaneity over budgets adds only a $\log\log K$ term to the number of questions (Proposition~\ref{prop:joint}).
\item \emph{The price at one benchmark} (Theorem~\ref{thm:instance}). At a fixed benchmark every valid audit pays $K/\eps+\Gamma/\eps^2$, where $\Gamma$ is the variance of one answer's influence under the best allocation of answers to questions, and an audit that learns that allocation attains $\Gamma/\eps^2$ up to a logarithmic factor as $\eps\to0$. At a fixed precision no audit matches every benchmark's own optimum (Proposition~\ref{prop:priced}).
\item \emph{A paired audit} (Theorem~\ref{thm:paired}). Two fresh paths per question and an exponential inequality for pairs of Bernoulli draws (Lemma~\ref{lem:paired}) give anytime-valid simultaneous bands whose penalty measures the within-question variance without a pilot. An exact-binomial final look with a corrected level (Lemma~\ref{lem:edgecp}) ends the audit at a fixed round. Run beside the multilevel audit behind Theorem~\ref{thm:frontier}, it gives one audit with the minimax rate at no more than $32/31$ of its own cost, plus $K$ answers (Corollary~\ref{cor:portfolio}).
\item \emph{Other curves and settings} (Section~\ref{sec:beyond}). The same machinery certifies pass@$k$ and majority voting, populations of questions, and trajectories whose answers depend on earlier ones. When score percentiles are unknown, the minimax mean squared error of a single question's curve is asymptotically $4/e$ times its value with known percentiles whenever generated answers, not labels, are the bottleneck (Theorem~\ref{thm:oracle}).
\item \emph{Evidence} (Section~\ref{sec:exp}). On 185 held-out pools the paired audit uses 0.74, 0.66 and 0.53 times the generated answers of the cheapest competing certified audit at $K=64$, 256 and 1024, and it is cheaper on all eight generator--benchmark answer sets. It missed none of its 3,255 runs against the exact curves. Its cost follows the three terms of the theory ($R^2=0.995$), and a cost law fitted before a new study predicted that study's bill within 0.6\%.
\end{itemize}

\paragraph{Outline.} Section~\ref{sec:model} sets up the problem. Sections~\ref{sec:fresh}--\ref{sec:instance} derive the price of a certified curve, first for fresh questions, then on a fixed list, then at a single benchmark. Section~\ref{sec:audit} gives the paired audit and Section~\ref{sec:beyond} extends it. Section~\ref{sec:exp} tests it on stored score pools and on a newly generated study. Long proofs are in Appendices~\ref{app:proofs}--\ref{app:oracle} and experimental details in Appendix~\ref{app:details}.

\section{The problem}\label{sec:model}

\paragraph{Questions, answers and scores.} A benchmark is a known list $\mathcal X=\{1,\ldots,M\}$ of questions with equal weights. At question $x$ a generated answer carries a numerical verifier score $S$ and a correctness bit $Y$. The pairs $(S,Y)$ of different answers are independent and identically distributed given $x$, with an unknown law $P_x$. The score can be a reward model, a unit-test count or the model's own mean token log-probability; nothing below depends on where it comes from. A path is a sequence of answers at one question, and $W_k$ is the correctness of the first answer attaining the highest score among its first $k$. The target is the best-of-$k$ curve
\begin{equation}\label{eq:target}
\begin{aligned}
\theta_k&=\frac1M\sum_{x=1}^Mp_k(x),\qquad p_k(x)=\E(W_k\mid x),\\
&1\le k\le K.
\end{aligned}
\end{equation}
Two other curves live on the same list. Pass@$k$ counts a question as solved when some answer among the first $k$ is correct \citep{codex}. Majority voting counts it as solved when the most frequent valid answer among the first $k$ is correct, with plurality ties broken uniformly at random and no valid answer counted as incorrect \citep{selfconsistency}.

\paragraph{Audits and what they pay.} An audit may choose questions, stop paths, query correctness and stop, all adaptively. It must return intervals $I_k$ of width at most $2\eps$ with
\[
\Prb\{\theta_k\in I_k\text{ for all }k\le K\}\ge1-\delta
\]
for \emph{every} answer law. It pays in three currencies: generated answers $N$, correctness queries $m$ (one per queried answer occurrence), and question visits $n$ (one per path started). Answers usually dominate the bill, since each needs a forward pass of the model, while a query may be cheap when the grader is a string comparison. We report all three. Unless stated otherwise $\delta=0.05$ and $\eps=1/32$.

\paragraph{What a band buys.} Simultaneity is what makes a budget choice safe. The following facts use nothing but coverage, so they hold whatever rule picked the budget after seeing the band.

\begin{proposition}[Decisions from a band]\label{prop:decisions}
Suppose $[l_k,u_k]$ covers $\theta_k$ for every $k\le K$, and $u_k-l_k\le2\eps$.
\begin{enumerate}[leftmargin=*,label=(\roman*)]
\item The set $\{k:u_k\ge\max_jl_j\}$ contains every budget that maximizes $\theta_k$.
\item Every $k$ with $l_k\ge\max_ju_j-\tau$ has $\theta_k\ge\max_j\theta_j-\tau$; the budget that maximizes $l_k$ qualifies with $\tau=2\eps$.
\item For known costs $c(k)$ and a price $\lambda\ge0$, a maximizer $\widehat k$ of $(l_k+u_k)/2-\lambda c(k)$ has $\max_k\{\theta_k-\lambda c(k)\}-\{\theta_{\widehat k}-\lambda c(\widehat k)\}\le2\eps$.
\item If a second system has a band $[l'_k,u'_k]$ covering its curve $\theta'_k$, then $l_k>u'_k$ implies $\theta_k>\theta'_k$, and both bands cover at once with probability at least $1-\delta-\delta'$.
\end{enumerate}
\end{proposition}

\begin{proof}
If $k^*$ maximizes $\theta$, then $u_{k^*}\ge\theta_{k^*}\ge\theta_j\ge l_j$ for every $j$, which is (i). For (ii), $\theta_k\ge l_k\ge\max_ju_j-\tau\ge\max_j\theta_j-\tau$; and $\max_kl_k\ge\max_j(u_j-2\eps)$. For (iii), write $c_k=(l_k+u_k)/2$, so $|c_k-\theta_k|\le\eps$; for any $k$, $\theta_k-\lambda c(k)\le c_k+\eps-\lambda c(k)\le c_{\widehat k}+\eps-\lambda c(\widehat k)\le\theta_{\widehat k}+2\eps-\lambda c(\widehat k)$. Part (iv) is a union bound.
\end{proof}

Pointwise intervals give none of this: a budget that looks best among $K$ pointwise intervals is best by selection as much as by accuracy.

\paragraph{Two kinds of variance.} Write
\begin{equation}\label{eq:variance}
\begin{aligned}
\vbar_k&=\frac1M\sum_xp_k(x)\{1-p_k(x)\},\\
\sumw&=\sum_{j=1}^K\max_{j\le k\le K}\vbar_k .
\end{aligned}
\end{equation}
For a question drawn uniformly from the list, the law of total variance splits the variance of $W_k$ as
\begin{equation}\label{eq:split}
\theta_k(1-\theta_k)=\vbar_k+\Var_x p_k(x),
\end{equation}
a within-question part and a between-question part. The second part is what random questions cost and balanced visits avoid. The tail sum $\sumw$ charges the answer at position $j$ of a path for the noisiest budget from $j$ on that still needs it. If every budget had within-question variance $v$, then $\sumw=Kv$. When the selected answer at a question becomes reliably right or reliably wrong as $k$ grows, $\sumw$ is much smaller: on the MMLU-Pro study of Section~\ref{sec:luna}, $\sumw=2.57$ against a worst case of $K/4=16$ at $K=64$.

\section{Fresh questions: the population price}\label{sec:fresh}

Start with the setting that most evaluation work assumes. Each path begins at a fresh question drawn from a population, and the target is the population average $\theta_k=\E_Xp_k(X)$. Every draw then carries the between-question variance, and the price is the following (proof in Appendix~\ref{app:joint}).

\begin{proposition}[Fresh questions]\label{prop:joint}
Let $0<\eps\le1/32$, $0<\delta\le1/16$ and $J=1+\lfloor\log_{16}K\rfloor$. Under deterministic resource caps, with an abort counted as a failure, every valid audit needs
\[
\begin{aligned}
n&\gtrsim\frac{\log(J/\delta)}{\eps^2},\qquad
m\gtrsim\frac{J\log(J/\delta)}{\eps^2},\\
N&\gtrsim\frac{K\log(1/\delta)}{\eps^2},
\end{aligned}
\]
and one predetermined record audit attains all three orders together.
\end{proposition}

The three currencies behave very differently. Generation is where the curve is expensive: every question costs about $K$ answers, because the largest budget needs a full path. Labels are cheap. A path needs correctness only at its score records, the answers that beat every earlier score, and when scores do not tie, the $k$th answer is a record with probability $1/k$ whatever the question \citep{renyi}; so a path of length $K$ needs about $H_K=\sum_{j\le K}1/j\approx\log K$ labels, and the lower bound shows that a factor of that size is unavoidable. Questions are nearly free: simultaneity over all $K$ budgets adds only $\log J\approx\log\log K$ to the $\log(1/\delta)$ that a single budget pays.

The reason is a geometric fact about winners. The winners at budgets $k\le\ell$ coincide unless the best of the first $\ell$ answers comes after position $k$, which happens with probability at most $1-k/\ell$. So
\[
\E\{(W_k-W_\ell)^2\mid x\}\le1-\frac k\ell\le\log\frac\ell k,
\]
and budgets within a factor $e^{r^2}$ of each other move together to within $r$ in root mean square. At resolution $r$ the curve has only about $\log K/r^2$ distinct budgets, and a chaining bound turns that count into a $\log\log K$ price for the supremum over all of them. A fixed list changes everything else.

\section{A fixed list of questions}\label{sec:theory}

Estimating every budget separately would cost $K$ times one budget. It does not, because two winners on the same path rarely disagree.

\begin{lemma}[Nested winners]\label{lem:coupling}
For $a\le k\le2a$, under first-maximum selection and with ties allowed,
\[
\E\{(W_k-W_a)^2\mid x\}\le p_a(x)\{1-p_a(x)\}.
\]
\end{lemma}

\begin{figure*}[t]
\centering\includegraphics[width=\textwidth]{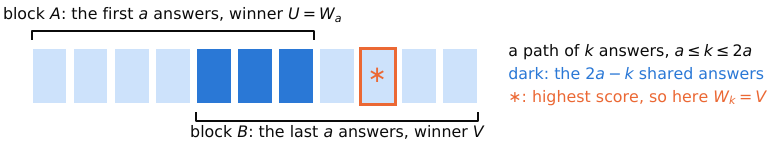}
\caption{The coupling behind Lemma~\ref{lem:coupling}. The highest score of the whole path lies in block $A$ or in block $B$, so the winner at budget $k$ is the winner of one of the two blocks. Exchanging the unshared parts of the blocks swaps their winners and leaves the law of the path unchanged.}
\label{fig:nested}
\end{figure*}

\begin{proof}
Fix $x$ and suppose first that the highest score on any path is attained once, almost surely. Take a path of length $k$ with $a\le k\le 2a$. Call its first $a$ positions block $A$ and its last $a$ positions block $B$; they share the $2a-k$ positions in the middle (Figure~\ref{fig:nested}). Let $U$, $V$ and $W$ be the correctness of the winners of $A$, of $B$ and of the whole path. The winner of the whole path is the winner of $A$ or the winner of $B$, so $W\in\{U,V\}$. If $U=V$, both agree with $W$; if $U\ne V$, exactly one of them does. Hence $\Prb(W\ne U)+\Prb(W\ne V)=\Prb(U\ne V)$. Exchanging the unshared parts of $A$ and $B$ does not change the law of the path, leaves the winner of the whole path in place, and swaps $U$ with $V$. The two probabilities on the left are therefore equal, and $\Prb(W\ne U)=\tfrac12\Prb(U\ne V)$.

Given the shared positions, $U$ and $V$ are independent with a common conditional mean $g$. Hence $\E(UV)=\E g^2$, $\E U=\E V=\E g=p_a(x)$, and
\[
\begin{aligned}
\Prb(W\ne U)
&=\tfrac12\bigl(\E U+\E V-2\E UV\bigr)\\
&=p_a(x)\{1-p_a(x)\}-\Var(g)\\
&\le p_a(x)\{1-p_a(x)\}.
\end{aligned}
\]
Since $W=W_k$ and $U=W_a$, this is the claim. When $k=2a$ the blocks are disjoint, $g$ is constant and equality holds.

For ties, attach to every answer an independent uniform key and rank by score, then by key. Given the scores, the correctness bits of different answers are independent, and answers with equal scores share one conditional correctness law, so the selected correctness has the same law under this rule as under first-maximum selection; in particular $p_a(x)$ is unchanged. If the path of length $k$ reaches a score strictly above the maximum of its first $a$ positions, the two winners carry different scores under either rule, and the disagreement event has the same probability under both. If the two maxima are equal, first-maximum selection keeps the same answer, so $W_k=W_a$, while the randomized rule may move to another tied answer. The disagreement probability under first-maximum selection is therefore at most the one under the randomized rule, which the argument above bounds.
\end{proof}

Budgets a factor two apart thus differ by at most the within-question noise of the smaller one, and none of the between-question variance enters. The rest of the theory is about how this noise adds up along a curve. Let $\mathcal C(M,K,v,s)$ be the laws with continuous conditional score distributions, $\max_k\vbar_k\le v$ and $\sumw\le s$, where $0\le v\le1/4$ and $0\le s\le Kv$. Put $M_\eps=\min(M,\eps^{-2})$, $a_*=\min(v,s)$ with $a_*\{1+\log(s/a_*)\}=0$ when $a_*=0$, $H_K=\sum_{j\le K}1/j$ and $\kappa_\delta=\kl(1-\delta,\delta)$, the Kullback--Leibler divergence between Bernoulli laws with means $1-\delta$ and $\delta$. Audits must stay valid for \emph{all} laws, inside the class and outside it.

\begin{theorem}[Price of a certified curve]\label{thm:frontier}
Let $0<\delta\le1/4$ and $0<\eps\le1/32$.
\begin{enumerate}[leftmargin=*,label=(\roman*)]
\item Over $\mathcal C(M,K,v,s)$ the minimax expected number of generated answers is
\begin{equation}\label{eq:frontN}
\widetilde\Theta\Bigl(\frac K\eps+M_\eps+\frac s{\eps^2}\Bigr).
\end{equation}
\item For every valid audit there is a law in $\mathcal C(M,K,v,s)$ at which its expected number of correctness queries is at least a constant, depending only on $\delta$, times
\begin{equation}\label{eq:frontm}
\frac{H_K}\eps+M_\eps+\frac{a_*\{1+\log(s/a_*)\}}{\eps^2}.
\end{equation}
\item One audit, told neither $v$ nor $s$, attains both \eqref{eq:frontN} and \eqref{eq:frontm}, up to factors that are polynomial in $\log\{K/(\eps\delta)\}$.
\end{enumerate}
\end{theorem}

\paragraph{Three obstructions.} Each term of \eqref{eq:frontN} is forced by a different kind of law (Appendix~\ref{app:frontier}).
\begin{itemize}[leftmargin=*]
\item \emph{Calibration.} A valid audit must allow for a rare answer that scores above everything seen so far and carries the minority label. If such an answer turns up with probability $3\eps/K$ per draw, $\theta_K$ moves by more than $2\eps$, so the audit has to generate about $K/\eps$ answers before it can rule the answer out. This term is present even when correctness is deterministic, and it is paid at every law: every valid audit generates at least $7\kappa_\delta K/(32\eps)$ answers in expectation, whatever the law.
\item \emph{Telling questions apart.} Give each question a hidden correctness bit that scores do not reveal. An interval of width $2\eps$ for the average needs most of the bits once $M\le\eps^{-2}$, and about $\eps^{-2}$ of them otherwise.
\item \emph{Within-question noise.} A law whose winners change label within a question forces $s/\eps^2$ answers, and the same law forces the label term $a_*\{1+\log(s/a_*)\}/\eps^2$: one law makes both resources expensive at once.
\end{itemize}
Labels follow the same pattern with $1/\eps$ in place of $K/\eps$: in $\widetilde\Theta$ form they cost $1/\eps+M_\eps+a_*/\eps^2$, and (ii) shows that the factors $H_K$ and $1+\log(s/a_*)$ cannot be removed. The first of them is paid at every law as well. For every continuous-score law, every valid audit uses at least $(1+\lfloor\log_{16}K\rfloor)\,3\kappa_\delta/(32\eps)$ correctness queries in expectation, one $3\kappa_\delta/(32\eps)$ for each of the disjoint score scales on which the winners of budgets $1,16,256,\ldots$ live \citep{fitas2026}. Labels for the whole curve cost what labels for $\theta_1$ alone cost, up to these logarithms: a law whose within-question variance halves with each budget already forces $\kappa_\delta a_*/(26\eps^2)$ queries for $\theta_1$. The larger budgets are paid for in generated answers, not in grading. If correctness stays noisy after the score is seen, with conditional success probability between $\eta$ and $1-\eta$, the price of labels rises to $(1+\lfloor\log_{16}K\rfloor)\kappa_\delta\eta c_0^2/(36\eps^2)$ at every such law once $\eps\le\eta c_0/6$, with $c_0=e^{-1/4}-e^{-4}$.

\paragraph{An example.} On our MMLU-Pro study ($M=100$, $K=64$, $\eps=1/32$) the three terms are $2{,}048$, $100$ and $2{,}628$. With the unrestricted envelope $s=K/4$ the last term would be $16{,}384$: the measured within-question variance is six times smaller than the worst case, and the last term decides the bill. These are rate terms, not a predicted bill; logarithmic factors and the start-up of Section~\ref{sec:audit} multiply them. Over the 185 held-out pools, $\sumw$ is a quarter of its worst case at $K=64$ and 6\% of it at $K=1024$.

\paragraph{The multilevel audit.} The upper bound in (iii) is a multilevel Monte Carlo audit \citep{giles2015,rheeglynn2015}. It estimates $\theta_1$ and adds dyadic corrections: at level $\ell$ a path of length $b_\ell=\min(2^\ell,K)$ contributes $W_{\min(k,b_\ell)}-W_{2^{\ell-1}}$ to every budget $k>2^{\ell-1}$ and nothing to smaller budgets. The corrections telescope to $\theta_k-\theta_1$. Lemma~\ref{lem:coupling} bounds the second moment of each correction by $\vbar_{2^{\ell-1}}$, and stopping each level on its own makes its sample size follow its variance. Summing over levels turns these variances into $\sumw$. The audit identifies the rate; the audit we run in practice is the one in Section~\ref{sec:audit}.

\section{One benchmark at a time}\label{sec:instance}

Theorem~\ref{thm:frontier} is a statement about classes of laws. At a fixed precision no audit can be optimal law by law, and the reason is a guess-and-verify argument \citep[Claim~1.8]{narayanan2024}: an audit that guesses a benchmark's correctness table can check its guess with about $\log(2/\delta)/\eps$ labels, and no single audit is that cheap at every table.

\begin{proposition}[No audit is optimal at every law]\label{prop:priced}
Let $\delta\le1/4$, $\eps\le1/128$, $K=1$, $M=\lfloor\pi/(1024\eps^2)\rfloor$ and $q=\lceil\log(2/\delta)/\{-\log(1-\eps)\}\rceil$. For every valid audit there is a law, with continuous scores and deterministic correctness, at which its expected numbers of generated answers, correctness queries and question visits are each at least $M/4$, while another valid audit uses exactly $q$ of each. The ratio $M/(4q)$ is of order $1/\{\eps\log(2/\delta)\}$.
\end{proposition}

\begin{proof}
\emph{A cheap audit for one table.} Fix $f:\{1,\ldots,M\}\to\{0,1\}$ with average $\bar f$. The audit $A_f$ draws $q$ questions independently and uniformly from the list, generates one answer at each and queries its correctness. If every value agrees with $f$, it returns $[\bar f-\eps,\bar f+\eps]$; otherwise it runs a fixed valid audit at level $\delta/2$, such as a fixed Hoeffding design, and returns its interval. Let $d=M^{-1}\sum_x\Prb(Y\ne f(x)\mid x)$. If $d\le\eps$, then $|\theta_1-\bar f|\le M^{-1}\sum_x|p_1(x)-f(x)|=d\le\eps$ and the first interval covers. If $d>\eps$, the first interval is returned with probability $(1-d)^q\le(1-\eps)^q\le\delta/2$. So $A_f$ is valid at every law, and at a law with $Y=f(x)$ it uses exactly $q$ answers, $q$ queries and $q$ visits.

\emph{Every audit is expensive at some table.} Let scores be independent and uniform on $[0,1]$, let $Y=f(x)$, and draw the $M$ bits $f(x)$ independently and uniformly. Scores, unqueried answers and the audit's own randomness carry no information about the bits, so after any transcript the bits of questions without a queried answer are independent and uniform given the transcript. Let $u$ be their number. Given the revealed bits, $\theta_1$ is a known shift of $\Bin(u,1/2)/M$, and an interval of width $2\eps$ contains at most $2\eps M+1$ of its possible values. If $u\ge M/2$, the interval therefore contains $\theta_1$ with conditional probability at most
\[
\begin{aligned}
&(2\eps M+1)\max_j\Prb\{\Bin(u,1/2)=j\}\\
&\le(2\eps M+1)\sqrt{\frac2{\pi u}}\\
&\le4\eps\sqrt{\frac M\pi}+\frac2{\sqrt{\pi M}}\\
&\le\frac18+\frac2{\sqrt{50\pi}}<\frac12,
\end{aligned}
\]
because $M\le\pi/(1024\eps^2)$ and $M\ge\lfloor\pi\cdot128^2/1024\rfloor=50$. Averaging the coverage requirement over the prior gives $1-\delta\le\Prb(u<M/2)+\tfrac12\Prb(u\ge M/2)$, so more than $M/2$ bits are revealed with probability at least $1-2\delta\ge1/2$, and the expected number $D$ of revealed bits is at least $M/4$ under the prior. Some table $f$ attains $\E_fD\ge M/4$. Each revealed bit needs an answer generated at its question, a correctness query and a visit, so all three expected counts are at least $M/4$ at that law, where $A_f$ uses $q$ of each. Finally, $q\le1+\log(2/\delta)/\eps$ and $M\ge\pi/(1024\eps^2)-1$ make $M/(4q)$ of order $1/\{\eps\log(2/\delta)\}$.
\end{proof}

What survives is the leading term as $\eps\to0$ at a fixed benchmark. For budget $k$ let $\widetilde W_k$ be the correctness of the highest-scoring answer among $k$, averaged over tied maxima, so that $\E(\widetilde W_k\mid x)=p_k(x)$. Let
\[
v_k(x)=k^2\Var\{\E(\widetilde W_k\mid Z_1,x)\mid x\},
\]
where $Z_1$ is one of the $k$ answers. This is the variance per answer of the all-subsets average of $\widetilde W_k$ over many answers at $x$, the first term of its Hoeffding decomposition \citep{hoeffding1948}. If a fraction $a_x$ of $N$ answers goes to question $x$, the benchmark estimate at budget $k$ has variance close to $\sum_xv_k(x)/(M^2a_xN)$, and
\begin{equation}\label{eq:gamma}
\Gamma=\min_{a}\max_{k\le K}\frac1{M^2}\sum_{x=1}^M\frac{v_k(x)}{a_x},
\end{equation}
with $a$ ranging over the simplex, is the smallest worst-budget variance per answer.

\begin{theorem}[Price at one benchmark]\label{thm:instance}
Let $0<\delta<1/2$ and fix a law on the list.
\begin{enumerate}[leftmargin=*,label=(\roman*)]
\item For $0<\eps\le1/32$ every valid audit has $\E N\ge(\kappa_\delta/128)(K/\eps+\Gamma/\eps^2)$, and every family of valid audits has $\liminf_{\eps\to0}\eps^2\E N\ge\kappa_\delta\Gamma/2$.
\item One valid audit, whose schedule does not depend on the law, has
\[
\limsup_{\eps\to0}\eps^2\E N\le2\Gamma\log(2K/\delta).
\]
\item $\Gamma\le\max_kk\vbar_k\le\sumw$.
\end{enumerate}
\end{theorem}

Part (i) is a bound at every law rather than at a worst case, and (ii) matches its leading term up to the factor $4\log(2K/\delta)/\kappa_\delta$. The audit behind (ii) spends a pilot to bound each $v_k(x)$, gives question $x$ answers in proportion to $\{\sum_k\lambda_kv_k(x)\}^{1/2}$ for a least favorable mixture $\lambda$ of budgets, as in Neyman allocation \citep{cochran1977,carpentier2015}, and averages each question's answers over all $k$-subsets. It labels every answer (Appendix~\ref{app:instance}). By (iii), $\sumw/\Gamma$ is what an audit gives up by reading one winner per path and visiting every question equally.

The gap is large on measured benchmarks (Figure~\ref{fig:gamma}). On the MMLU-Pro study $\Gamma=0.28$ against $\sumw=2.57$, and over the held-out pools the median of $\sumw/\Gamma$ is 12, with interquartile range 8.4 to 21.9. Most of it is allocation. The model is right, or wrong, at least 98\% of the time on 58 of the 100 MMLU-Pro questions, and the optimal allocation puts 90\% of its answers on 23 questions. Finding those questions costs answers, though. At $\eps=1/32$ the pilot of the audit behind (ii) alone would use 179,200 answers, more than twice the whole bill of the paired audit below, and Proposition~\ref{prop:priced} says that no audit finds them for free at a fixed precision.

\begin{figure*}[t]
\centering\includegraphics[width=\textwidth]{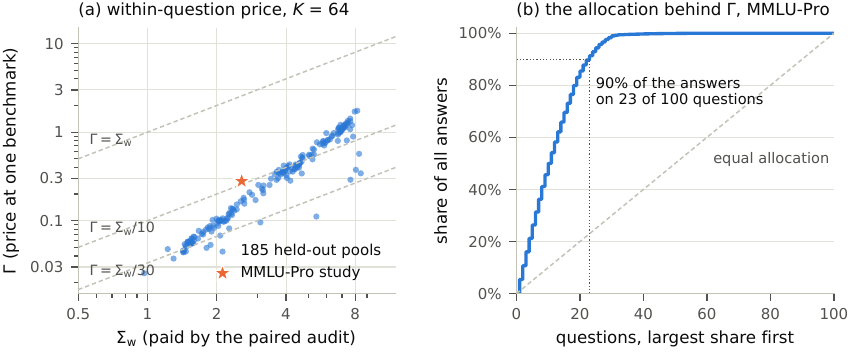}
\caption{The price at one benchmark. (a) $\Gamma$ of Theorem~\ref{thm:instance} against $\sumw$ at $K=64$ for the 185 held-out pools and the MMLU-Pro study (law of the first 4,000 reference answers per question). The median of $\sumw/\Gamma$ is 12. (b) The allocation of answers to questions that attains $\Gamma$ on the MMLU-Pro study, sorted by share.}
\label{fig:gamma}
\end{figure*}

\section{The paired audit}\label{sec:audit}

Stratified sampling removes the between-question variance only if the confidence sequence can measure the variance that is left. A plug-in estimate of each question's mean does this slowly. Early on, the error of the plug-in center is itself a between-question deviation, so the penalty charges exactly the variance that balance was supposed to remove. Two independent draws at the same question avoid the problem: their squared difference has mean exactly twice the within-question variance, whatever the question's mean.

\begin{lemma}[Paired exponential inequality]\label{lem:paired}
Let $A,B$ be independent $\Ber(p)$ variables, $0\le\lambda<1$ and $\psi(\lambda)=-\log(1-\lambda)-\lambda$. Then
\begin{equation}\label{eq:pairmgf}
\E\exp\{\pm\lambda(A+B-2p)-\psi(\lambda)(A-B)^2\}\le1 .
\end{equation}
\end{lemma}

\begin{proof}
Let $q=1-p$ and $u=2\lambda$. The pair $(A,B)$ takes the values $(0,0)$, $(1,1)$ and the two mixed values with probabilities $q^2$, $p^2$ and $2pq$, and $e^{\lambda-\psi(\lambda)}=e^{2\lambda}(1-\lambda)$. Enumerating,
\[
\begin{aligned}
&\E e^{\lambda(A+B-2p)-\psi(\lambda)(A-B)^2}\\
&=e^{-up}\bigl[q^2+p^2e^{u}+2pq\,e^{u}(1-u/2)\bigr]\\
&=e^{-up}\bigl[q^2+e^u\{p(1+q)-upq\}\bigr].
\end{aligned}
\]
It remains to show $e^{up}-q^2-e^u\{p(1+q)-upq\}\ge0$ for $u\ge0$. Expand in powers of $u$. The coefficient of $u^m/m!$ is $1-q^2-p(1+q)=0$ for $m=0$, and $p^m-p(1+q)+mpq$ for $m\ge1$, which vanishes for $m=1$ and $m=2$. For $m\ge3$ it equals
\[
\begin{gathered}
p^m+(m-2)p-(m-1)p^2\\
=p\bigl\{p^{m-1}+(m-2)\\
\qquad -(m-1)p\bigr\}\ge0,
\end{gathered}
\]
because $p^{m-1}=(1-q)^{m-1}\ge1-(m-1)q$ by Bernoulli's inequality. This proves the inequality with the sign $+$. Applying it to $1-A$ and $1-B$, which are $\Ber(q)$ and have the same squared difference, gives the sign $-$.
\end{proof}

Penalized exponential inequalities of this kind go back to \citet{fan2015} and are the basis of betting confidence sequences \citep{howard2021,betting}. Estimating a variance from squared differences is classical as well \citep{wangramdas2025}. What the pair buys here is exactness: for two Bernoulli draws the penalty needs no estimated center and no correction term, so the within-question variance enters the band from the first pair on.

\paragraph{The audit.} The audit runs in rounds, and each round visits every question once with a fresh path. Two rounds make a pair.
\begin{enumerate}[leftmargin=*]
\item \emph{Grow only what is needed.} In pair $j$, every path is grown to the largest budget whose interval is still wider than $2\eps$, and correctness is queried only at score records.
\item \emph{Update the band.} For each unresolved budget $k$, let $T_j$ be the sum of the $2M$ selected-correctness values of the pair and $D_j$ the number of questions whose two values disagree. Multiplying \eqref{eq:pairmgf} over questions and pairs gives two nonnegative supermartingales per budget, and Ville's inequality \citep{ville1939} with a union over the $2K$ sides gives, at every pair,
\begin{equation}\label{eq:pairedci}
\begin{aligned}
\theta_k&\in\frac{\sum_j\lambda_jT_j}{2M\sum_j\lambda_j}
\ \pm\
\frac{\sum_j\psi(\lambda_j)D_j+\log(2K/\alpha_S)}{2M\sum_j\lambda_j},\\
&\alpha_S=0.95\,\delta.
\end{aligned}
\end{equation}
The audit keeps the running intersection of these intervals.
\item \emph{Bet on the variance seen so far.} The bet $\lambda_j=\min\{\lambda_{\max},\eps/(\eps+\hat v_j)\}$ maximizes $\lambda\eps-\psi(\lambda)\hat v_j$, where $\hat v_j$ is half the disagreement rate of earlier pairs, started with one pseudo-pair of variance $1/4$ and floored at $\min(10^{-4},\eps/100)$.
\item \emph{Retire and stop.} A budget retires when its interval has width at most $2\eps$. The remaining $0.05\,\delta$ goes to an exact-binomial \citep{clopperpearson} interval at a fixed final round, the smallest even number of rounds at which every possible count gives an interval of width at most $2\eps$. That round ends the audit.
\end{enumerate}
Figure~\ref{fig:anatomy} follows one audit on one stored pool. Budgets resolve at different times, and paths shrink as they do: after four pairs only the smallest budgets are still open, and the last pairs generate a few answers per question.

\begin{figure*}[t]
\centering\includegraphics[width=\textwidth]{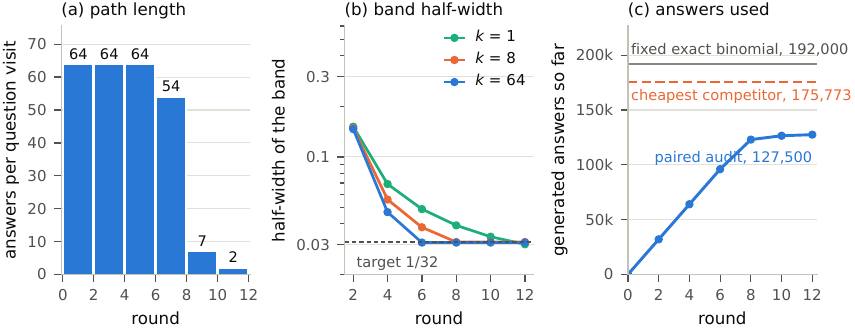}
\caption{One paired audit, traced round by round (the stored run on the MATH500 pool of Figure~\ref{fig:hero}a, 250 questions, $K=64$). (a) Answers generated per question visit in each pair of rounds; paths stop at the largest unresolved budget. (b) Half-width of the band at three budgets; a budget retires when it reaches $1/32$, and the final look is at round 18. (c) Answers used so far, against the fixed exact-binomial design and the cheapest of the competing certified audits on the same pool. The audit stops after 12 rounds with 127,500 answers, and its band covers the exact curve.}
\label{fig:anatomy}
\end{figure*}

\paragraph{A final look with unequal means.} The pooled count at the final round is a sum of independent Bernoulli variables whose average mean is $\theta_k$ but whose individual means differ, one per question. \citet[Theorem~4]{hoeffding1956} bounds $\Prb(S\le b)$ by the binomial tail when $b\le n\theta_k-1$. The counts just below the mean are not covered, and there the comparison can fail. With $n=2$ and means $(1,2\theta-1)$ for some $\theta>1/2$, the count is at most 1 with probability $2(1-\theta)$, more than the binomial value $1-\theta^2$. At $\theta=(1-a)^{1/2}$ an exact-binomial test at level $a$ rejects whenever the count is at most 1, which happens with probability $2\{1-(1-a)^{1/2}\}>a$. A small correction of the level repairs this.

\begin{lemma}[Exact-binomial tails for unequal Bernoulli sums]\label{lem:edgecp}
Let $B\sim\Bin(n,\theta)$ and let $S$ be a sum of $n$ independent Bernoulli variables whose means average to $\theta$. If $0<a<1/4$ and $a'=a/(1+a)$, a lower-tail binomial test at level $a'$ rejects with probability at most $a$ under $S$. The same holds for the upper tail.
\end{lemma}

\begin{proof}
Let $b$ be the largest count the test rejects, so $\Prb(B\le b)\le a'<1/4$. Suppose first $b\le n-2$, and suppose $b>n\theta-1$. With $h=b+1\le n-1$ we have $\theta<h/n$, and since binomial lower tails decrease in the success probability, $\Prb(B\le b)\ge\Prb\{\Bin(n,h/n)\le h-1\}$. The sum of $h+1$ independent $\Ber\{h/(h+1)\}$ variables, padded with $n-h-1$ zeros, has mean $h$, so Hoeffding's comparison at $h-1$, one below the mean, gives
\[
\begin{aligned}
&\Prb\{\Bin(n,h/n)\le h-1\}\\
&\ge\Prb\{\Bin(h+1,h/(h+1))\le h-1\}\\
&=1-\Bigl(\frac{h}{h+1}\Bigr)^h\frac{2h+1}{h+1}\ge\frac14,
\end{aligned}
\]
because the subtracted term decreases in $h$ from $3/4$ at $h=1$. This contradicts $\Prb(B\le b)<1/4$, so $b\le n\theta-1$, and Hoeffding's comparison gives $\Prb(S\le b)\le\Prb(B\le b)\le a'\le a$. If $b=n-1$, rejection means $1-\theta^n\le a'$, and a union bound over the $n$ variables gives $\Prb(S\le n-1)\le n(1-\theta)\le-n\log\theta\le-\log(1-a')=\log(1+a)\le a$. Applying the argument to the complements proves the upper tail.
\end{proof}

The final look inverts each of the $2K$ tails at $a/(1+a)$ with $a=0.05\,\delta/(2K)$. Since $a/(1+a)$ differs from $a$ by a factor $1+a\le1+10^{-3}$, the final round is the same as with the uncorrected level in every configuration we ran.

\begin{theorem}[Paired audit]\label{thm:paired}
For every fixed $\lambda_{\max}\in(0,1)$, the paired audit has simultaneous coverage at least $1-\delta$ for every answer law and stops by a deterministic round cap. With $L$ logarithmic in $K$, $1/\eps$ and $1/\delta$, and constants depending on $\lambda_{\max}$,
\begin{equation}\label{eq:pairedcost}
\begin{aligned}
\E N&=O\!\left(KM+L\Bigl[\frac K\eps+\frac{\sumw}{\eps^2}\Bigr]\right),\\
\E m&\le H_K\,\E n .
\end{aligned}
\end{equation}
\end{theorem}

\begin{proof}[Proof of coverage]
Fix a budget $k$ and write $A_{xj}$, $B_{xj}$ for the selected-correctness values of the two paths at question $x$ in pair $j$. The bet $\lambda_j$ is fixed before either path of the pair is generated, and the paths of different questions are independent. Multiplying \eqref{eq:pairmgf} over questions and then over pairs shows that
\[
\exp\Biggl\{\begin{aligned}
&\sum_{j\le J}\lambda_{j}\Bigl(\sum_x(A_{xj}+B_{xj})-2M\theta_k\Bigr)\\
&-\sum_{j\le J}\psi(\lambda_{j})\sum_x(A_{xj}-B_{xj})^2
\end{aligned}\Biggr\}
\]
is a nonnegative supermartingale in $J$, and so is its mirror image. Ville's inequality at level $\alpha_S/(2K)$ for each of the $2K$ processes, and a union bound, give \eqref{eq:pairedci} at all pairs at once with probability at least $1-\alpha_S$, and running intersections keep this event. A budget that has retired needs no further values: attach a latent full path to every scheduled visit, and note that the audit reads the winner at budget $k$ only while $k$ is active, so no unobserved value enters the sequence. At the final round, Lemma~\ref{lem:edgecp} bounds each of the $2K$ tails by $a=\alpha_C/(2K)$ with $\alpha_C=0.05\,\delta$. The final interval is intersected with the betting interval; if the intersection is empty, which happens only outside the coverage event, a point is returned. Coverage holds with probability at least $1-\alpha_S-\alpha_C=1-\delta$.
\end{proof}

The cost bound is proved in Appendix~\ref{app:pairedcost}. The calculation behind it is short. Replace the penalty in \eqref{eq:pairedci} by its mean and suppose the bet uses the true $\vbar_k$. With $\lambda=\eps/(\eps+\vbar_k)$, the penalty part of the radius is $\psi(\lambda)\vbar_k/\lambda\le\eps/2$, and the other part, $L/(n\lambda)$ after $n$ question visits, falls below $\eps/2$ once
\begin{equation}\label{eq:nk}
n_k=2L\,(\vbar_k+\eps)/\eps^2,\qquad L=\log(2K/\alpha_S).
\end{equation}
An answer at position $j$ is generated only while some budget $k\ge j$ is unresolved. After a start-up of one complete pair of rounds, $2KM$ answers, the audit therefore generates about $\sum_j\max_{k\ge j}n_k=2L(K/\eps+\sumw/\eps^2)$ answers. Section~\ref{sec:costlaw} fits this shape to measured costs and finds coefficients $2.47$, $1.38$ and $1.75$ where the calculation gives $2$, $2$ and $2$. The start-up is small when $M\eps^2$ is small; Section~\ref{sec:losses} shows where it is not.

\paragraph{Explicit constants.} The constants hidden in \eqref{eq:pairedcost} come from the adaptive bet. A variant that also runs fixed dyadic bets turns the calculation into a proof with small constants.

\begin{proposition}[Paired audit with a grid of bets]\label{prop:grid}
Suppose the audit also runs, for every budget, the bets $\lambda_h=2^{-h-2}$ for $h=0,\ldots,H-1$ with $H=1+\lceil\log_2\{(1/4+\eps)/\eps\}\rceil$, each giving the interval $T/n\pm\{\psi(\lambda_h)D+L'\}/(\lambda_hn)$ after $n$ rows of complete pairs, where $T$ and $D$ are pooled success and disagreement counts and $L'=\log\{2K(H+1)/\alpha_S\}$, and intersects all its intervals. Let $n_F$ be the number of rows at the final round. Then for every $\eta\in(0,1)$,
\[
\begin{aligned}
\E N\le\min\Bigl\{Kn_F,\ &2KM+\bigl\{16L'+\log(K/\eta)\bigr\}\\
&\Bigl(\frac K\eps+\frac{\sumw}{\eps^2}\Bigr)+\eta Kn_F\Bigr\}.
\end{aligned}
\]
\end{proposition}

\begin{proof}
Fix $k$ and write $v=\vbar_k$. The grid contains a bet with $\eps/\{8(v+\eps)\}<\lambda\le\eps/\{4(v+\eps)\}\le1/4$, and $\psi(\lambda)\le\lambda^2$ there. The count $D$ is a sum of independent Bernoulli variables with total mean $nv$, so $\E e^{D/2}\le\exp\{(e^{1/2}-1)nv\}$ and $\Prb(D>2nv+2u)\le e^{-u}$. Outside that event the radius of this bet is at most
\[
2\lambda v+\frac{2\lambda u}n+\frac{L'}{\lambda n}\le\frac\eps2+\frac1n\Bigl\{\frac{\eps u}{2(v+\eps)}+\frac{8(v+\eps)L'}\eps\Bigr\},
\]
which is at most $\eps$ once $n\ge(v+\eps)(16L'+u)/\eps^2$. Apply this at the $K$ deterministic row counts $n_k$ given by that bound, rounded up to a multiple of $2M$, with $u=\log(K/\eta)$. Outside an event of probability $\eta$, every budget $k$ is resolved within $n_k$ rows or at the final round. An answer at position $j$ is generated only while some budget $k\ge j$ is unresolved, so $N\le\sum_j\max_{k\ge j}n_k\le2KM+\{16L'+\log(K/\eta)\}\sum_j\max_{k\ge j}(\vbar_k+\eps)/\eps^2$, and $\sum_j\max_{k\ge j}(\vbar_k+\eps)=\sumw+K\eps$. Outside that event, $N\le Kn_F$.
\end{proof}

Taking $\eta=\eps$ gives the order of \eqref{eq:pairedcost} with the constant 16 in place of the calculation's 2.

\paragraph{One audit for the rate and for practice.} The start-up $KM$ is absent from Theorem~\ref{thm:frontier}, so the paired audit attains the rate \eqref{eq:frontN} only when $M\lesssim1/\eps+s/(K\eps^2)$. The multilevel audit attains the rate everywhere, but with poor constants. Running the two side by side keeps the best of both.

\begin{corollary}[Portfolio]\label{cor:portfolio}
Fix $\rho\in(0,1)$. Run the paired audit at level $(1-\rho)\delta$ and the multilevel audit of Theorem~\ref{thm:frontier}(iii) at level $\rho\delta$ on separate fresh answers, one path at a time, always advancing the audit whose answer count divided by its weight, $1-\rho$ or $\rho$, is smaller, and return the band of the first to finish. The result is a valid audit, and on every run
\[
N\le\min\Bigl\{\frac{N_{\rm paired}}{1-\rho},\frac{N_{\rm multilevel}}{\rho}\Bigr\}+K,
\]
where $N_{\rm paired}$ and $N_{\rm multilevel}$ are the answers each audit would use alone on the same answers.
\end{corollary}

\begin{proof}
Write $w_1=1-\rho$ and $w_2=\rho$, and let $N_i$ be the answers audit $i$ has used so far. Each step generates one path, of at most $K$ answers, for the audit with the smaller ratio $N_i/w_i$. After every step, $N_j\le w_jN_i/w_i+K$ for $i\ne j$: at the last step at which audit $j$ advanced, $N_j/w_j\le N_i/w_i$; since then $N_i$ has not decreased, and that step added at most $K$ answers to $N_j$. Hence $N=N_1+N_2\le N_i/w_i+K$ for both $i$. Each audit's state depends only on its own answers, so until the first completion $N_i$ is at most the count the audit would use alone, and the bound follows. Each audit covers every $\theta_k$ with probability at least $1-w_i\delta$ whatever the other does, so both cover with probability at least $1-\delta$ and the returned band covers. Both stop by deterministic caps, so the portfolio does too.
\end{proof}

With $\rho=1/32$ the portfolio attains \eqref{eq:frontN} up to a factor 32 and uses at most $32/31$ of the answers of the paired audit at level $31\delta/32$, plus $K$. The bound concerns generated answers only.

\section{Other curves, other settings}\label{sec:beyond}

\subsection{Pass@\texorpdfstring{$k$}{k} and majority voting}\label{sec:ustat}

Coverage of the paired audit uses only that each path yields a correct-or-incorrect outcome at every budget, computed from the first $k$ answers, and that paths are independent given the question. Pass@$k$ and majority voting have that form, so the paired audit certifies them unchanged; only the label count differs, since pass@$k$ needs every answer before the first correct one checked and majority voting needs a grade for each distinct valid answer.

A path can also serve every budget at once through its subsets. A path of length $b\ge k$ contains $\binom bk$ subsets of size $k$, and each is distributed as $k$ fresh answers. Averaging a curve's outcome over all of them gives an unbiased estimate $U_{x,k}$ of $p_k(x)$, a U-statistic \citep{hoeffding1948}. For best-of-$k$, sorting the path by score gives the average in closed form: a group of $d$ tied answers with $c$ answers strictly below it contributes its average correctness times $\{\binom{c+d}k-\binom ck\}/\binom bk$, which matches first-maximum selection in expectation \citep{webgpt,gao}. For pass@$k$ with $c$ correct answers on the path, $U_{x,k}=1-\binom{b-c}k/\binom bk$ \citep{codex}. For majority voting we average over random subsets, an incomplete U-statistic that is still unbiased.

These averages are bounded but not binary, so the paired inequality does not apply. A bounded-mean inequality does: for $Z\in[-1,1]$ and $0\le\lambda<1$, $e^{\lambda z-\psi(\lambda)z^2}\le1+\lambda z$ on $[-1,1]$, hence $\E e^{\lambda(Z-\E Z)-\psi(\lambda)Z^2}\le1$. With $Z=U_{x,k}-c_{x,k}$ for a prediction $c_{x,k}\in[0,1]$ built from earlier rounds, multiplying over the questions of a round and over rounds gives two nonnegative supermartingales per budget, and Ville's inequality gives simultaneous bands. We call this the all-subsets audit. It needs the correctness of every answer on the path, records or not, so it uses more labels than the paired audit unless one grade can serve every repeat of an answer (Section~\ref{sec:luna}).

\paragraph{First-success stopping for pass@\texorpdfstring{$k$}{k}.} For pass@$k$ alone a path can stop at its first correct answer. With $H=\min\{j:Y_j=1\}$, or $K+1$ if none of the first $K$ answers is correct, the indicator $\ind\{H\le k\}$ is the pass@$k$ outcome of the path at every $k\le K$. On $n$ paths at questions drawn independently and uniformly from the list, the stopping positions are independent and identically distributed, and the Dvoretzky--Kiefer--Wolfowitz inequality with the constant of \citet{massart1990} gives one band of radius $\{\log(2/\delta)/(2n)\}^{1/2}$ for every $k$ at once \citep{dkw1956}. A path at question $x$ uses $\sum_{j<K}\{1-p_1(x)\}^j$ answers in expectation and needs the correctness of each of them.

\subsection{One grade per distinct answer}

On multiple-choice or short-answer benchmarks the grader is a deterministic function of the question and the extracted answer. Reusing one grade for every repeat of a question--answer pair changes no observation of any audit: bands, coverage and answer counts are identical on every run, and only the number of distinct grader calls falls. The same holds for any curve computed from the extracted answers. In the MMLU-Pro study below, fewer than 400 distinct grader calls certify each curve.

\subsection{A population of questions}\label{sec:panel}

The target so far is the listed questions. For a population of unseen questions the between-question variance returns, and Proposition~\ref{prop:joint} gives its price. A fixed list can still serve a population. Draw a panel of $m$ questions independently from the population, audit the panel as a fixed list at level $\delta$, and add
\[
r_m=\Bigl\{\frac{\log(2K/\alpha)}{2m}\Bigr\}^{1/2}
\]
to every radius. By Hoeffding's inequality \citep{hoeffding1963} and a union over budgets and sides, the panel average of $p_k$ is within $r_m$ of the population value at every $k$ with probability at least $1-\alpha$, so the widened band covers the population curve with probability at least $1-\delta-\alpha$. The audit then pays the within-question price for the panel and the population pays once, through $m$.

\subsection{Answers that depend on earlier answers}\label{sec:dependent}

The theory above assumes that answers on a path are independent given the question. Agents that revise, search or reuse context break this. Coverage of a curve does not need independence within a path. Let $Z_1,\ldots,Z_n$ be independent complete trajectories from one fixed generation policy whose budgets are prefixes of each other, so that later answers may depend on earlier ones. Let $W_k(Z)$ be the correctness of the first highest-scoring answer among the first $k$, and
\[
\begin{aligned}
V&=\sum_{k=2}^K\Prb\{W_k\ne W_{k-1}\},\\
R&=\E\{\text{number of strict score records in }Z\}.
\end{aligned}
\]
Since the winner changes only at a strict record, $0\le V\le R-1\le K-1$.

\begin{proposition}[Bands for dependent trajectories]\label{prop:dependent}
For a universal constant $C$,
\[
\E\max_{k\le K}\Bigl|\frac1n\sum_{i=1}^nW_k(Z_i)-\E W_k\Bigr|\le C\sqrt{\frac{\log(2+V)}n},
\]
and with probability at least $1-\delta$ the maximum exceeds this bound by at most $\sqrt{\log(1/\delta)/(2n)}$. With $B$ independent vectors $\sigma^{(b)}$ of Rademacher signs and
$A_B=B^{-1}\sum_{b\le B}\max_{k\le K}|n^{-1}\sum_i\sigma_i^{(b)}W_k(Z_i)|$, the radius
\[
2A_B+2\sqrt{\frac{\log(3/\delta)}{2B}}+3\sqrt{\frac{\log(3/\delta)}{2n}}
\]
covers every $\E W_k$ simultaneously with probability at least $1-\delta$, without knowledge of $V$.
\end{proposition}

\begin{proof}
Let $\tau_1=0$ and $\tau_k=\sum_{j=2}^k\Prb(W_j\ne W_{j-1})$. For $0<r\le1$, group consecutive budgets by $\lfloor\tau_k/r^2\rfloor$; there are at most $1+\lfloor V/r^2\rfloor$ groups. Within a group $[a,b]$, the pointwise minimum and maximum of the $W_k$ form a bracket with
$\E\{\max_{a\le k\le b}W_k-\min_{a\le k\le b}W_k\}^2\le\sum_{j=a+1}^b\Prb(W_j\ne W_{j-1})<r^2$. The bracketing entropy integral is at most $\int_0^1\sqrt{\log(2+V/r^2)}\,dr\le\sqrt{\log(2+V)}+\int_0^1\sqrt{2\log(1/r)}\,dr$, and the bracketing maximal inequality for bounded classes \citep[Theorem~2.14.2]{vdvw} gives the bound in expectation. Changing one trajectory moves the maximum by at most $1/n$, so the bounded-differences inequality \citep{mcdiarmid1989} gives the high-probability statement. For the computable radius, symmetrization bounds the expected maximum error by twice the expected Rademacher supremum. Both the maximum error and the conditional expected Rademacher supremum change by at most $1/n$ when one trajectory changes, and two bounded-difference events account for $3\sqrt{\log(3/\delta)/(2n)}$. Given the trajectories, each simulated supremum lies in $[0,1]$, and Hoeffding's inequality adds $2\sqrt{\log(3/\delta)/(2B)}$. A union bound over the three events finishes the proof.
\end{proof}

A complete trajectory costs $K$ answers and, when generation needs no labels, at most its number of records in queries. The proposition gives coverage and a sample size from any upper bound on $V$. It does not optimize acquisition, and it covers one fixed policy.

\subsection{The price of unknown score percentiles}\label{sec:oracle}

Every audit above treats the verifier score as an ordinal quantity: only the order of the scores on a path matters. If the population distribution of scores at a question were known, each answer's percentile would be known too, and labels could be spent exactly where the winners of each budget land \citep{fitas2026}. How much is that knowledge worth? We answer for one question and a mean-squared-error target, where the answer is exact.

All answers are independent and identically distributed from one score--correctness law. In the \emph{oracle} experiment the auditor is given the score distribution, so it observes each answer's percentile; in the \emph{unknown} experiment it observes scores only. In both it may pool answers across budgets and acquire up to $C$ answers and $T$ correctness labels adaptively, each answer labeled at most once. Let
\[
R=\inf_{\mathcal A}\sup_P\max_{k\le K}\E_P(\widehat\theta_k-\theta_k)^2
\]
be the minimax worst-budget mean squared error, with $R_{\rm oracle}$ and $R_{\rm unknown}$ its values in the two experiments. We allow caps that hold on every run and caps that hold in expectation uniformly over laws.

\begin{theorem}[Unknown score percentiles]\label{thm:oracle}
If $K\to\infty$, $T/\log K\to\infty$, $C/K\to\infty$ and $C\ge T$, then under either kind of cap
\[
\begin{aligned}
R_{\rm oracle}&=(1+o(1))\max\Bigl\{\frac{\log K}{16T},\frac K{8C}\Bigr\},\\
R_{\rm unknown}&=(1+o(1))\max\Bigl\{\frac{\log K}{16T},\frac K{2eC}\Bigr\}.
\end{aligned}
\]
Hence $R_{\rm unknown}/R_{\rm oracle}\to4/e$ when
\[
\limsup C\log K/(KT)\le2,
\]
and $R_{\rm unknown}/R_{\rm oracle}\to1$ when
\[
\liminf C\log K/(KT)\ge8/e.
\]
\end{theorem}

\begin{figure*}[t]
\centering\includegraphics[width=.92\columnwidth]{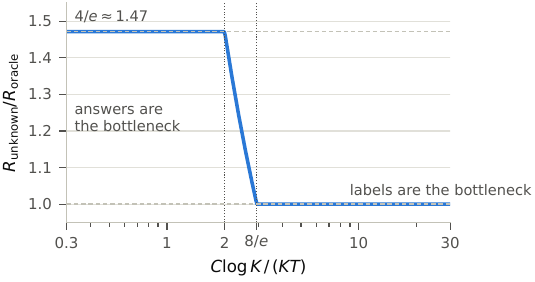}
\caption{Theorem~\ref{thm:oracle}: the leading ratio of the minimax errors without and with known score percentiles, as a function of the balance $C\log K/(KT)$ between generated answers $C$ and labels $T$.}
\label{fig:percentiles}
\end{figure*}

With $x=C\log K/(KT)$, the label term dominates the oracle rate when $x\ge2$ and the unknown rate when $x\ge8/e$, which gives the two limits and the curve of Figure~\ref{fig:percentiles}. An estimator that uses only the observed order of the scores attains the unknown rate. The lower bounds allow any adaptive acquisition and delayed labeling. So $4/e\approx1.47$ is the exact asymptotic price of not knowing the percentiles when answers are the bottleneck. When labels are, the knowledge is worth nothing to first order. The proof, in Appendix~\ref{app:oracle}, rests on a variance bound for the complete-pool U-statistic that keeps every order of overlap between subsets and holds uniformly over score--correctness laws,
\begin{equation}\label{eq:ustatbound}
\Var(U_{n,k})\le\frac k{2en}+\Bigl(\frac kn\Bigr)^2+\frac1n,
\end{equation}
where $U_{n,k}$ averages the selected label over all subsets of size $k$ of $n$ answers. The same constant $4/e$ appears in the work of \citet{fitas2026} with a different meaning: there it compares an envelope design with the variance-optimal design, both with known percentiles.

\subsection{Spending labels when percentiles are known}\label{sec:design}

When the percentiles are known, a band of fixed width calls for a different allocation of labels than a small mean squared error does. Let $P$ be a known $K\times d$ matrix whose row $k$ gives the probability that the budget-$k$ winner falls in each of $d$ score cells. For a proposal $q$ on the cells, draw $T$ independent cells, query their labels, and estimate
\[
\widehat\theta_k=\tfrac12+\frac1T\sum_{i=1}^T\frac{P_{k,I_i}}{q_{I_i}}\Bigl(Y_i-\tfrac12\Bigr).
\]
With $V_k(q)=\sum_iP_{ki}^2/q_i$, $W_k(q)=\max_iP_{ki}/q_i$ and $x=\log(2K/\delta)$, Bernstein's inequality gives each budget the radius
\[
r_k(T,q)=\frac{(W_k+1)x}{6T}+\sqrt{\frac{V_kx}{2T}+\Bigl(\frac{(W_k+1)x}{6T}\Bigr)^2},
\]
and $r_k\le\eps$ holds exactly when $T\ge x\{V_k(q)/(2\eps^2)+(W_k(q)+1)/(3\eps)\}$. The smallest sufficient number of labels is therefore governed by the convex objective
\[
C_\rho(q)=\max_k\{V_k(q)+\rho W_k(q)\},\qquad\rho=2\eps/3,
\]
not by $V_k$ or $W_k$ alone. Writing $B_{(k,i),j}=P_{kj}^2+\rho P_{ki}\ind\{j=i\}$ gives $C_\rho(q)=\max_{k,i}\sum_jB_{(k,i),j}/q_j$, and convex duality gives the identity
\[
\min_{q\in\Delta_d}C_\rho(q)=\max_{\lambda\in\Delta_{Kd}}\Bigl[\sum_j\Bigl(\sum_{k,i}\lambda_{ki}B_{(k,i),j}\Bigr)^{1/2}\Bigr]^2,
\]
with the inner minimum at $q_j$ proportional to the square root for fixed $\lambda$. Strict feasibility and divergence at any zero coordinate give strong duality, so any feasible $q$ and any $\lambda$ bracket the optimum. Table~\ref{tab:design} compares three proposals on two complete pools. The fixed-width proposal saves 8.50\% and 4.74\% of the labels of the envelope proposal $q\propto\max_kP_{k\cdot}$, and 3.26\% and 11.68\% of those of the variance-optimal proposal. The variance objective $\max_kV_k$ can have several near-minimizers with different label counts, and that column reports the one our solver returns.

\begin{table*}[t]
\caption{Sufficient label draws for a band of half-width $1/32$ at level 95\% over $K=100$ budgets with known percentiles, on two complete HumanEval+ pools scored by Llama-3.1-70B unit tests \citep{coderm,evalplus}. The designs receive the same pool; answer generation is not charged. Last column: relative gap between the feasible fixed-width proposal and its dual lower bound.}\label{tab:design}
\centering\small
\begin{tabular}{lrrrr}
\toprule
solutions from & envelope & variance-optimal & fixed-width & primal--dual gap\\
\midrule
Llama-3-8B & 6,681 & 6,319 & 6,113 & 0.215\%\\
Llama-3-70B & 5,398 & 5,822 & 5,142 & 0.233\%\\
\bottomrule
\end{tabular}
\end{table*}

\section{Experiments}\label{sec:exp}

\subsection{Stored score pools}\label{sec:pools}

The primary evidence is 185 score pools from \citet{weaver2025}. Eight answer sets (GPQA, MATH500, MMLU and MMLU-Pro questions, each answered by Llama-3.1-8B-Instruct and Llama-3.1-70B-Instruct, with 100 answers per question) are each scored by 19 to 27 reward models and verifiers; one answer set and one scorer make one pool. An audit draws answers uniformly with replacement from a question's stored answers, so every pool defines an exact answer law whose curve we compute in closed form, ties included, and every miss is visible. The questions of each benchmark were split in half once, before any tuning. Twenty-six settings of stratified betting audits, three paired and twenty-three per-row, were compared on the first halves and the best was frozen. The held-out pools use the other halves, 250 to 360 questions each. Thirty-two further pools (Weaver's combined score, CodeRM unit tests, CodeContests) come from sources seen during development and are reported separately. We run five audits per pool at $K\in\{64,256,1024\}$ with $\eps=1/32$ and $\delta=0.05$. Appendix~\ref{app:details} gives licenses, the protocol and every competing design.

The curves themselves are flat where it matters (Figure~\ref{fig:plateau}): the median held-out curve comes within $1/32$ of its best value by budget 10, 126 of the 185 by budget 16 and 174 by budget 64. At $K=64$, no curve among all 217 pools sits more than $1/32$ below its best value over budgets up to 64.

\begin{figure*}[t]
\centering\includegraphics[width=.92\columnwidth]{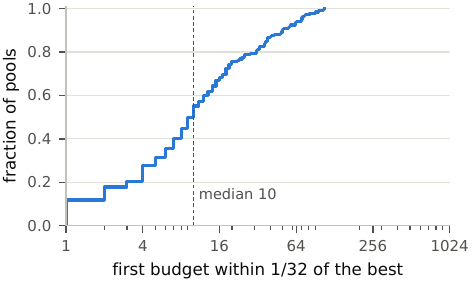}
\caption{First budget at which the exact curve of a held-out pool comes within $1/32$ of its best value over budgets up to 1024.}
\label{fig:plateau}
\end{figure*}

\paragraph{The competition.} The comparator on each pool is the cheapest of three certified audits, chosen with hindsight on that pool: an exact-binomial audit that stops revealing paths once every budget is settled, nested exact-binomial looks with budget retirement, and a rank-based stopping audit (Appendix~\ref{app:comparators}). All three draw questions at random. At $K=64$ and 256 we also ran a fixed Hoeffding design, a fixed exact-binomial design and the record design of \citet{fitas2026}; none of them is cheaper than the best of the three on any of the 370 held-out cases.

\subsection{The held-out comparison}\label{sec:heldout}

\begin{table*}[t]
\caption{Paired audit on the 185 held-out pools, relative to the cheapest competing certified audit on each pool (median ratio of five-run means). Brackets: 95\% bootstrap interval over the eight answer sets. Labels are compared with the cheapest competitor in labels.}\label{tab:pools}
\centering\small
\begin{tabular}{rcccc}
\toprule
$K$ & answers & interval & pools cheaper & labels\\
\midrule
64 & 0.74 & [0.67, 0.88] & 175/185 & 0.73\\
256 & 0.66 & [0.59, 0.73] & 175/185 & 0.68\\
1024 & 0.53 & [0.47, 0.58] & 173/185 & 0.61\\
\bottomrule
\end{tabular}
\end{table*}

\begin{figure*}[t]
\centering\includegraphics[width=\textwidth]{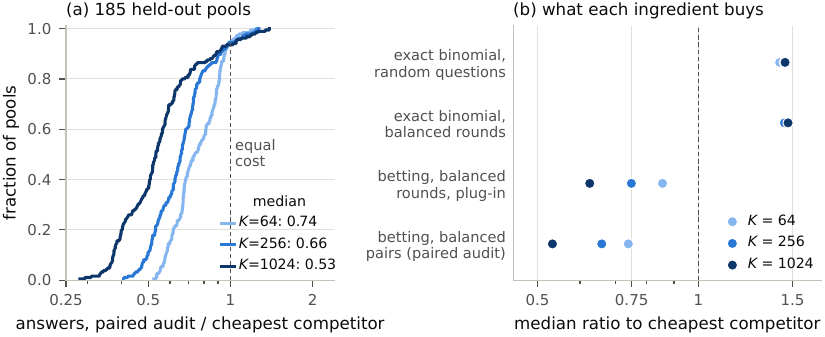}
\caption{(a) Distribution of the ratio of generated answers, paired audit to the cheapest competing certified audit, over the 185 held-out pools. Left of the dashed line the paired audit is cheaper than every competitor on that pool. (b) Median ratio for the designs of the ablation in Section~\ref{sec:ablation}; every design uses the same retirement of resolved budgets.}
\label{fig:heldout}
\end{figure*}

Table~\ref{tab:pools} and Figure~\ref{fig:heldout}a give the result. The paired audit is cheaper on at least 173 of 185 pools at every horizon and on all eight answer sets, and the saving grows with $K$. It also needs fewer question visits (median ratios 0.84, 0.72 and 0.67) and fewer labels. Across all 217 stored pools, including the 32 seen during development, it missed none of its 3,255 runs against the exact curves. The competitors missed 17 times in their 13,020 runs, all within their $\delta$. On the 32 development pools the median ratios are 0.52, 0.43 and 0.36.

\subsection{Where the saving comes from}\label{sec:ablation}

Figure~\ref{fig:heldout}b takes the audit apart. Balance alone buys nothing: nested exact-binomial looks cost the same on balanced rounds as on random questions (ratio of the two 1.01, 1.00 and 1.03), because an exact-binomial interval cannot see a smaller variance. A stratified betting sequence with per-question plug-in centering \citep{betting} captures much of the saving, and the paired inequality takes the rest. The paired audit beats that sequence, and a version of it with the paired audit's bet and cap, on every one of the 185 pools at every horizon; against the better of the two its median saving is 9\%, 11\% and 17\%.

\subsection{When the audit loses, and why the saving grows}\label{sec:losses}

\begin{figure*}[t]
\centering\includegraphics[width=\textwidth]{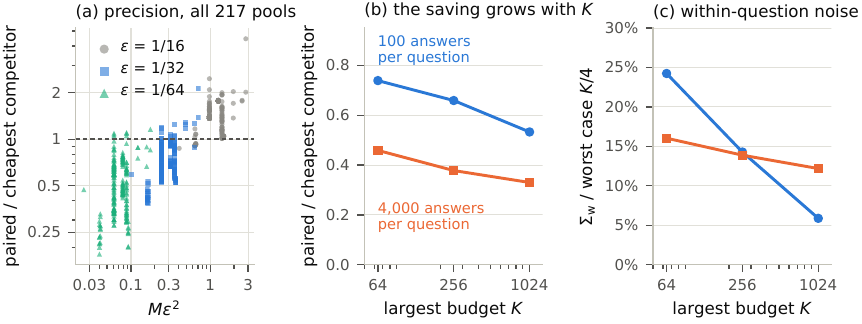}
\caption{(a) Cost ratio at $K=64$ for every stored pool and three precisions, against $M\eps^2$; the start-up of two complete rounds decides the ratio at the right. (b) Median ratio on the held-out pools, whose questions have 100 stored answers, and on a pool with 4,000 answers per question built from the MMLU-Pro reference stream. (c) $\sumw$ as a fraction of its worst case $K/4$ in the same two settings (held-out: median).}
\label{fig:regimes}
\end{figure*}

The ten held-out pools where the paired audit loses at $K=64$ are those with little variance to remove. Their median within-question share of variance at budget 64 is 0.45, against 0.24 for the rest, and across all 185 pools the Spearman correlation between that share and the cost ratio is 0.82.

Precision matters in the way the cost law says (Figure~\ref{fig:regimes}a). We reran the frozen audit at $K=64$ with $\eps=1/16$ and $\eps=1/64$, after recording the prediction that the start-up of two complete rounds would dominate at coarse precision and fade at fine precision. At $\eps=1/16$ two complete rounds already exceed what a fixed design needs, and the paired audit loses on every held-out pool (median ratio 1.40). At $\eps=1/64$ it wins on 179 of 185 (median 0.51).

Why does the saving grow with $K$? Part of the answer is the pools. A stored question has only 100 answers, so at budget $k$ its top-scored answer appears with probability $1-0.99^k$: 0.47 at $k=64$ and 0.92 at 256. Large budgets then redraw a small set, and $\sumw$ falls from 24\% of its worst case at $K=64$ to 6\% at $K=1024$ (Figure~\ref{fig:regimes}c). To separate this effect from the audit, we built a pool from the first 4,000 reference answers per question of the MMLU-Pro study below, where the top answer appears with probability 0.016, 0.062 and 0.226 at the three horizons and $\sumw$ stays between 12\% and 16\% of its worst case. The paired audit still wins by more at larger $K$. Its ratio to the cheapest competitor is 0.46, 0.38 and 0.33 (Figure~\ref{fig:regimes}b), with no miss in 75 runs.

\subsection{The cost law}\label{sec:costlaw}

\begin{figure*}[t]
\centering\includegraphics[width=\textwidth]{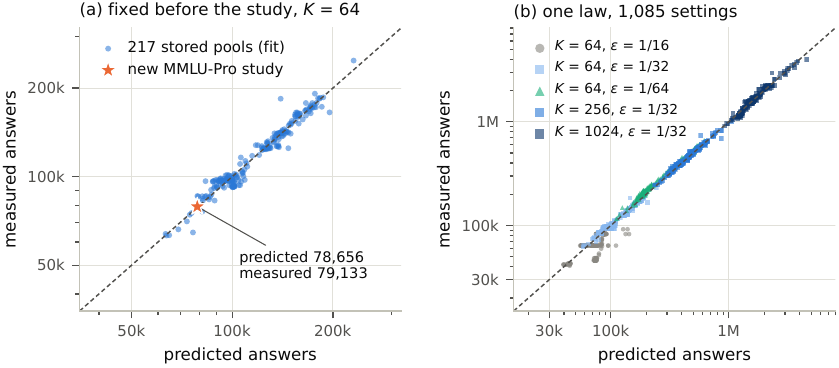}
\caption{The cost law. (a) $N\approx2.16\,KM+15.13\,K/\eps+12.88\,\sumw/\eps^2$, fitted to two paired audits per stored pool at $K=64$ and $\eps=1/32$ ($R^2=0.958$), against measured cost. The star is the new MMLU-Pro study: the coefficients were fixed before any of its answers existed, and its $\sumw$ comes from the separate reference stream. (b) One law $aKM+L(bK/\eps+c\,\sumw/\eps^2)$ for all horizons and precisions ($R^2=0.995$).}
\label{fig:costlaw}
\end{figure*}

The audit's cost follows the three terms of Theorem~\ref{thm:paired}. Across 1,085 pool, horizon and precision settings, a nonnegative least-squares fit of $aKM+L(bK/\eps+c\,\sumw/\eps^2)$ with $L=\log(2K/\alpha_S)$ has $R^2=0.995$ and median relative error 4.7\% (Figure~\ref{fig:costlaw}b). Replacing $\sumw$ by $K\max_k\vbar_k$, the worst-case envelope of the same variance, lowers $R^2$ to 0.906 and raises the median error to 18.1\%, so the tail sum is the right quantity. Fitted at $K=64$ alone, the law predicts the settings with $K=256$, $K=1024$ and $\eps=1/64$ with median errors 2.3\%, 5.2\% and 7.0\%, and those with $\eps=1/16$, where the start-up dominates, with 11.3\%.

The step that matters most is prospective (Figure~\ref{fig:costlaw}a). An earlier fit, on two audits per pool at $K=64$, gave $N\approx2.16\,KM+15.13\,K/\eps+12.88\,\sumw/\eps^2$ ($R^2=0.958$, median error 2.7\%), and these coefficients went into the analysis of the MMLU-Pro study before any of its answers was generated. Given the study's $\sumw=2.57$, computed from its separate reference stream, the law predicted 78,656 answers; the audit used 79,133. Resampling the 100 question-level reference curves 2,000 times puts the prediction's 5--95\% range at 69,096 to 89,152 answers.

\subsection{A newly generated study}\label{sec:luna}

We fixed the protocol first: 100 MMLU-Pro questions \citep{mmlupro2024} in mathematics, physics, chemistry and engineering; the \texttt{gpt-6-luna} model at temperature 1 with eight answers per request; the mean token log-probability of an answer as its verifier score; three independent audit streams and a separate reference stream. The model produced 716,384 answers for \$17.24, 409,448 of them for the reference. Every audit reads its stream in generation order, so a run costs exactly what it would have cost with answers generated on demand. Table~\ref{tab:luna} gives the cost of each certified curve and Figure~\ref{fig:mmlupro} the bands of the first stream.

\begin{table*}[t]
\caption{MMLU-Pro study, $M=100$, $K=64$, $\eps=1/32$, $\delta=0.05$: means over three audit streams. Grader calls reuse one grade per distinct question--answer pair; dollars are the generation cost at the price we paid per answer. The designs below the rule certify best-of-$k$ with fixed sample sizes.}\label{tab:luna}
\centering\small
\begin{tabular}{llrrrr}
\toprule
curve & audit & answers & labels & grader calls & cost (\$)\\
\midrule
best-of-$k$ & paired & 79,133 & 6,160 & 256 & 1.91\\
best-of-$k$ & all-subsets & 78,933 & 55,590 & 385 & 1.91\\
pass@$k$ & paired & 60,467 & 5,943 & 285 & 1.46\\
pass@$k$ & all-subsets & 49,067 & 49,067 & 385 & 1.18\\
majority voting & paired & 62,267 & 58,999 & 356 & 1.50\\
majority voting & all-subsets & 49,067 & 49,067 & 385 & 1.18\\
\midrule
best-of-$k$ & fixed exact binomial & 192,000 & & & 4.61\\
best-of-$k$ & record design \citep{fitas2026} & 257,216 & & & 6.18\\
best-of-$k$ & fixed Hoeffding & 262,400 & & & 6.31\\
\bottomrule
\end{tabular}
\end{table*}

\begin{figure*}[t]
\centering\includegraphics[width=\textwidth]{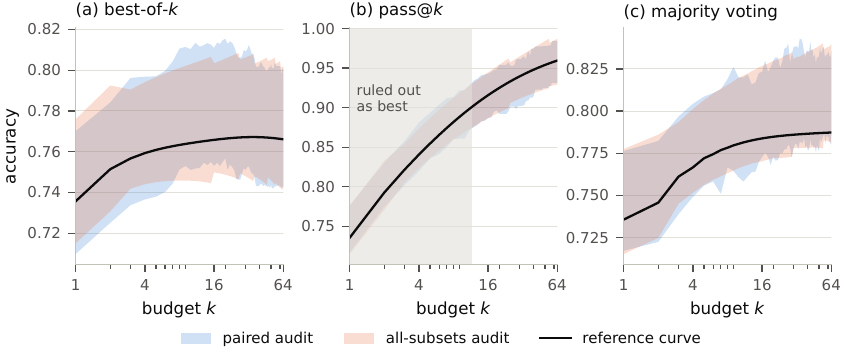}
\caption{Certified curves for 100 MMLU-Pro questions (first audit stream): simultaneous bands of half-width at most $1/32$ at level 95\% from the paired audit and the all-subsets audit, with reference curves estimated from 409,448 separately generated answers. Shaded in (b): budgets that the pass@$k$ bands rule out as a best budget.}
\label{fig:mmlupro}
\end{figure*}

The best-of-$k$ curve costs 79,133 answers, 41\% of the fixed exact-binomial design, and \$1.91 against \$4.61. Pass@$k$ and majority voting are cheaper still with the all-subsets audit, and because the grader is deterministic, fewer than 400 distinct grader calls certify each curve. All 18 bands contain the reference curves at every budget. The closest band edge is a majority-voting band at budget 55, $5.3\times10^{-4}$ from the reference value, below the reference's own bootstrap standard error of $7.9\times10^{-4}$ there.

\paragraph{What the bands say.} The three curves tell different stories. Pass@$k$ climbs from 0.736 to 0.960, and in every stream its bands rule out every budget below 12 as a best budget (Proposition~\ref{prop:decisions}(i)). Best-of-$k$, with the model's own log-probability as the verifier, gains only three points, from 0.736 to 0.767 on the reference curve, and majority voting reaches 0.787. Most of what sampling could buy on these questions is out of the verifier's reach, as \citet{stroebl2024} found for imperfect verifiers, and the bands certify it. In every stream, whichever audit certified each curve, the pass@$k$ band at budget 64 lies at least 0.125 above the best-of-$k$ band, so by Proposition~\ref{prop:decisions}(iv) pass@64 exceeds best-of-64 accuracy by more than 12 points with probability at least 90\%.

\floatbarrier
\section{Related work}\label{sec:related}

\paragraph{Test-time scaling.} Repeated sampling with a selection rule is a standard way to trade computation for accuracy. Pass@$k$ was introduced for code generation with an unbiased estimator \citep{codex}, large-scale sampling and filtering produced competition-level programs \citep{alphacode}, best-of-$n$ selection with a reward model has an unbiased estimator of its own \citep{webgpt,gao}, and majority voting over sampled reasoning paths is self-consistency \citep{selfconsistency}. \citet{largelanguagemonkeys} report how coverage grows with the number of samples, and \citet{weaver2025} combine weak verifiers to close part of the gap between generation and verification. \citet{stroebl2024} show that verifier errors bound what resampling can gain, which is why the budget choice matters. \citet{kazdan2025} predict pass@$k$ at large $k$ from a beta-binomial model of per-question success rates and give harder questions more samples, and \citet{mars2026} stop parallel reasoning traces once the majority vote is unlikely to change. These works report or predict curves; they do not certify them.

\paragraph{Statistics of model evaluation.} \citet{miller2024} splits the variance of benchmark accuracy by the law of total variance and recommends several answers per question. Its target is a population of unseen questions, which keeps the between-question term that a fixed list removes. Prediction-powered inference \citep{angelopoulos2023}, active testing \citep{kossen2021} and anytime-valid evaluation with e-processes \citep{celeus2026} save labels when estimating a single score; bands for tuning curves \citep{lourie2024} cover the best score among trials, not the correctness of a verifier's choice. Here the object is a whole curve produced by paths of generated answers, and generation, not labeling, is the main cost.

\paragraph{Sequential inference and sampling design.} Our bands are confidence sequences \citep{howard2021,betting} built from the penalized exponential inequality of \citet{fan2015} and Ville's inequality \citep{ville1939}. Estimating a variance from squared differences of pairs is classical \citep{wangramdas2025}; the paired inequality is its exact form for Bernoulli draws. Balanced rounds are stratified sampling with questions as strata, and the allocation behind $\Gamma$ is Neyman allocation for a least favorable mixture of budgets \citep{cochran1977}, learned from a pilot as in adaptive stratified sampling \citep{carpentier2015}. The upper bound of Theorem~\ref{thm:frontier} is multilevel Monte Carlo \citep{giles2015,rheeglynn2015} with a coupling of nested winners. The all-subsets estimates are U-statistics \citep{hoeffding1948}, and the label counts rest on the theory of records \citep{renyi}. The lower bounds use the change-of-measure argument for adaptive experiments \citep{kaufmann2016} and the guess-and-verify argument against instance optimality of \citet{narayanan2024}.

\paragraph{The closest work.} \citet{fitas2026} studies the same curve for a population of questions. That paper proves that auditing all widths up to $N$ needs labels in proportion to $1+\log N$ and generated answers in proportion to $N$, gives a known-percentile design and a record design with simultaneous bands, and asks whether simultaneous confidence needs an extra logarithm. Our label lower bounds reuse its disjoint score scales. For fresh questions and fixed-width bands under deterministic caps, Proposition~\ref{prop:joint} shows that simultaneity costs generated answers no extra logarithm, and for one question and mean squared error, Theorem~\ref{thm:oracle} gives the leading constants and the exact price of unknown percentiles. The fixed-list results of Sections~\ref{sec:theory}--\ref{sec:audit} have no counterpart there.

\section{Discussion}\label{sec:discussion}

A certified scaling curve on a fixed benchmark costs three things: calibration of the score tail, the identity of the questions, and within-question noise summed along the curve. The last term is where a random audit overpays, because most of the variance it pays for lies between questions, and on measured benchmarks the within-question sum is a fraction of its worst case. Revisiting questions in balanced pairs and retiring budgets as they resolve turns that into a practical audit, and a cost law with three terms tells in advance what it will spend.

Theorem~\ref{thm:instance} marks where the next saving lies. At a single benchmark the within-question term shrinks to $\Gamma$, about a tenth of $\sumw$ on our data, and most of that saving comes from not spending answers on questions the model almost always gets right or almost always gets wrong. An audit that learns this allocation reaches $\Gamma$ as the precision grows. At the precisions of Section~\ref{sec:exp} its pilot costs more than the paired audit's whole bill, and Proposition~\ref{prop:priced} rules out a free version, so the open problem is an audit that learns the allocation as it goes, without a separate pilot, and pays for learning only what it uses.

The results have a definite scope. The fixed-list target is the listed questions; for a population the between-question variance returns, and a random panel (Section~\ref{sec:panel}) is the way back. The paired audit pays a start-up of two complete rounds, which is why it loses at coarse precision on a large benchmark, and the cost law says in advance when that happens. The price of unknown percentiles is a mean-squared-error statement for one question. Dependent trajectories are covered for one fixed generation policy.

In practice the recommendation is short. Report a simultaneous band, not a curve; choose budgets from the band with Proposition~\ref{prop:decisions}; revisit every question in pairs of rounds; and use the cost law with a rough value of $\sumw$ to budget the audit before generating anything. On our 100-question MMLU-Pro study with 64 budgets, the best-of-$k$ band cost \$1.91 of generation.

\paragraph{Data.} The stored pools come from public releases: the Weaver collection on Hugging Face \citep{weaver2025}, released under the MIT license according to its dataset cards; the execution results released with the CodeRM repository \citep{coderm}; and the CodeContests test split \citep{alphacode}, whose DeepMind-provided non-code materials are CC BY 4.0 (third-party materials may have separate terms). MMLU-Pro \citep{mmlupro2024} is released under the MIT license. The MMLU-Pro answers were generated for this study with the protocol of Appendix~\ref{app:luna}.

\vspace{1.6ex plus .5ex minus .2ex}\noindent\begin{minipage}{\columnwidth}
\paragraph{Use of AI tools.} Generative AI tools assisted with code implementation, experimental design and analysis, literature review, figure preparation, and drafting and editing this manuscript. The authors are responsible for the claims, proofs, results, references, and final text.
\end{minipage}

\onecolumn
\appendix
\pdfbookmark[0]{Appendix}{appendix}
\begin{center}{\LARGE\bfseries Appendix}\end{center}
\vspace{1ex}
\noindent The appendix gives the proofs that do not appear in the main text (Appendices~\ref{app:proofs}--\ref{app:oracle}) and the details of the experiments (Appendix~\ref{app:details}).

\section{Proofs: conventions}\label{app:proofs}

Throughout, a valid audit returns intervals of width at most $2\eps$ that cover every $\theta_k$ at once with probability at least $1-\delta$ at every law. Two laws $P,Q$ with $|\theta_k(P)-\theta_k(Q)|>2\eps$ for some $k$ are separated by every valid audit: the event that its interval at budget $k$ contains $\theta_k(P)$ has probability at least $1-\delta$ under $P$ and at most $\delta$ under $Q$, so this binary test forces the Kullback--Leibler divergence between the two laws of the whole transcript to be at least $\kappa_\delta=\kl(1-\delta,\delta)$ \citep{kaufmann2016}. The audit's own choices carry no information about the law, so the chain rule charges each observation its conditional divergence. When $Q$ changes only the answer laws $P_x$, the transcript divergence is therefore at most $\sum_xC(x)\KL(P_x\|Q_x)$, where $C(x)$ is the expected number of answers generated at $x$ under $P$; revealing every correctness value for free can only increase it. We refer to this as testing.

\section{Proof of Proposition~\ref{prop:joint}}\label{app:joint}

In this appendix every path starts at a fresh question $X$ drawn independently from a population, and the target is $\theta_k=\E_Xp_k(X)$. The midpoint of any valid width-$2\eps$ interval estimates $\theta_k$ to within $\eps$, so a valid audit yields estimates with uniform error at most $\eps$ with probability $1-\delta$. Caps are deterministic, and an audit that reaches a cap without an answer has failed.

\subsection{A hard family}
Work with uniform score percentiles and the coordinate $t=-\log u$, in which the winner of $k$ answers has density $ke^{-kt}$. Put $J=1+\lfloor\log_{16}K\rfloor$, $k_j=16^j$ and $I_j=[1/k_j,2/k_j]$; these intervals are disjoint. For signs $v\in\{-1,1\}^J$ and $a=8\eps\le1/4$, let correctness at score coordinate $t$ be $\Ber(1/2+av_j)$ when $t\in I_j$ and $\Ber(1/2)$ otherwise. At the budgets $k_i$,
\begin{equation}\label{eq:signmatrix}
\theta(v)=\tfrac12\mathbf 1+aAv,\qquad A_{ij}=e^{-k_i/k_j}-e^{-2k_i/k_j}.
\end{equation}
The diagonal entries are $e^{-1}-e^{-2}>0.2325$. Using $e^{-z}-e^{-2z}\le z$, the entries for finer bands sum to at most $\sum_{d\ge1}16^{-d}=1/15$ in each row, and the entries for coarser bands to at most $e^{-16}/(1-e^{-16})$. Every row is therefore diagonally dominant by more than $0.16$, and $\lVert A^{-1}\rVert_\infty\le6.25$ \citep{varah}. An estimate of the curve with error at most $\eps$ recovers every sign by rounding $a^{-1}A^{-1}(\widehat\theta-\mathbf 1/2)$, because the coordinate error is at most $6.25\,\eps/a<1$.

\subsection{Correctness queries}
Give the signs a uniform prior and let the auditor choose, adaptively, a percentile at which to see a fresh correctness label, up to $m$ times; this is at least as informative as querying labels of generated answers. Action probabilities given the past carry no information about the signs, and a label in band $I_j$ depends only on $v_j$, so the posterior factorizes over coordinates. Let $L_j$ be the posterior log-odds of $v_j$, $e_j=1/(1+e^{|L_j|})\ge\frac12e^{-|L_j|}$ and $p(T)=\prod_j(1-e_j)$, the largest posterior probability of a sign vector. A valid audit decodes all signs with Bayes probability at least $1-\delta$, so $\E p(T)\ge1-\delta$, and by Markov's inequality $p(T)\ge1-2\delta$ with probability at least $1/2$. On that event $\sum_je_j\le-\log p(T)\le4\delta$ and $\sum_je^{-|L_j|}\le8\delta$, and Jensen's inequality gives
\begin{equation}\label{eq:oddsneed}
\E\sum_j|L_j|\ge\frac J2\log\frac J{8\delta}.
\end{equation}
Let $N_j$ count the labels seen in band $j$, $\ell=\log\{(1/2+a)/(1/2-a)\}\le8a$ and $d=2a\ell\le16a^2$. Under the true signs, $v_jL_j=dN_j+M_j$ with $M_j$ a martingale and $\E M_j^2\le\ell^2\E N_j$, and $\sum_jN_j\le m$. Hence
\begin{equation}\label{eq:oddssupply}
\E\sum_j|L_j|\le16a^2m+8a\sum_j\sqrt{\E N_j}\le16a^2m+8a\sqrt{Jm}.
\end{equation}
With $b=\log\{J/(8\delta)\}\ge\log2$ and $z=a^2m/J$, the two displays give $16z+8\sqrt z\ge b/2$, which fails when $z<b/1024$. Therefore
\[
m\ge\frac{J}{65536\,\eps^2}\log\frac J{8\delta}.
\]
The argument allows adaptive allocation and adaptive stopping under a deterministic label cap.

\subsection{Question visits}
Now let questions differ. Under signs $v$, question $i$ carries independent bits $X_{ij}\sim\Ber(1/2+av_j)$, $j<J$; an answer whose score coordinate lies in $I_j$ has correctness $X_{ij}$, and every other answer an independent fair bit. Answer pairs stay independent and identically distributed given the question, and averaging over questions gives exactly \eqref{eq:signmatrix}. Reveal every bit of every visited question for free. The data then split over coordinates, and each coordinate is a test between $\Ber(1/2+a)$ and $\Ber(1/2-a)$ from $n$ draws. By the Bretagnolle--Huber inequality \citep{tsybakov2009} and $\kl\{\Ber(1/2+a),\Ber(1/2-a)\}=2a\ell\le16a^2$, the Bayes error of one coordinate is at least $e_n=\frac14e^{-16a^2n}$. Decoding all coordinates succeeds with probability at most $(1-e_n)^J$; requiring at least $1-\delta$ gives $e_n\le1-(1-\delta)^{1/J}\le2\delta/J$, hence
\[
n\ge\frac1{1024\,\eps^2}\log\frac J{8\delta}.
\]
An actual audit sees no more per question, so the bound holds for it.

\subsection{Generated answers}
Two laws that differ only for score coordinates in the band $B=[1/K,2/K]$, where correctness is $\Ber(1/2\pm a)$, have best-of-$K$ targets that differ by $2a\int_{1/K}^{2/K}Ke^{-Kt}dt=2a(e^{-1}-e^{-2})>2\eps$ when $a=8\eps$. A generated answer lands in $B$ with probability at most $1/K$, and its label then has divergence at most $16a^2$. Reveal every label for free. With identical questions, the next answer has the same law whatever the audit does, so over a deterministic cap of $N$ answers, padded with uninformative draws after stopping, the transcript divergence is at most $16Na^2/K$, and the Bretagnolle--Huber inequality gives $N\ge cK\eps^{-2}\log(1/\delta)$. This bound uses one hard bit. A union over the $J$ bands does not raise it, because a single stream of answers visits all bands at once, and we claim no extra $\log J$ for generated answers.

\subsection{The record audit}\label{app:records}
\paragraph{Records.} Order the $K$ answers of a path by generation time and let $R_k$ be the rank of answer $k$ among the first $k$. The map from orderings to $(R_1,\ldots,R_K)$ is a bijection onto $\prod_k\{1,\ldots,k\}$, so when scores do not tie the ranks are independent and uniform, and the record indicators $\ind\{R_k=k\}$ are independent $\Ber(1/k)$, whatever the question \citep{renyi,stepanov}. With ties, the strict records are a subset of the records of the order broken by independent uniform keys, which have this law, so every bound below on the number of queries still holds. Querying the first answer and every later record, and carrying the last queried label forward, gives $W_k$ for every $k$. For $n$ paths the average $\widehat\theta_k$ of $W_k$ is unbiased with variance at most $1/(4n)$, and the number of queries $Q$ has mean $nH_K$ and variance at most $nH_K$; by Bernstein's inequality $Q\le nH_K+\sqrt{2nH_Kx}+2x/3$ with probability $1-e^{-x}$.

\paragraph{Uniform accuracy.} For $k\le\ell$ the winners at $k$ and $\ell$ coincide when the best of the first $\ell$ answers lies among the first $k$, which has probability at least $k/\ell$; so $\E(W_k-W_\ell)^2\le1-k/\ell\le\log(\ell/k)$. Cut $[0,\log K]$ into pieces of length at most $r^2$ and group the budgets whose logarithms fall in the same piece. For a group, the pointwise minimum and maximum of its $W_k$ bracket every member, and $\E(\max-\min)^2\le r^2$: if the earliest and the latest budget of the group select the same answer, every budget between them selects it too, so the bracket is nonzero only if the winner changes inside the group, which has probability at most $1-k_{\min}/k_{\max}\le r^2$. So the class $\{W_k:k\le K\}$ has $L_2$ bracketing number at most $2+\lceil\log K/r^2\rceil$. The bracketing maximal inequality \citep[Theorem~2.14.2]{vdvw} with envelope $1$ gives $\E\sup_k|\widehat\theta_k-\theta_k|\le C\sqrt{\log(2+\log K)/n}$, since
$\int_0^1\sqrt{\log(2+\log K/r^2)}\,dr\le\sqrt{\log(2+\log K)}+\int_0^1\sqrt{2\log(1/r)}\,dr$. Changing one path moves the supremum by at most $1/n$, so McDiarmid's inequality adds $\sqrt{\log(1/\delta)/(2n)}$ with probability $1-\delta$.

\paragraph{Nested paths.} Let $D=\lceil\log_2K\rceil$, $b_0=1$, $b_j=\min(2^j,K)$, and
\[
\delta_j=\frac{\delta}{4(D+1)}+\delta2^{j-D-2},\qquad n_j=\bigl\lceil C\eps^{-2}\log(2/\delta_j)\bigr\rceil,\qquad 0\le j\le D.
\]
The first $n_j$ paths are extended to length $b_j$. The budgets in $(b_{j-1},b_j]$ span a logarithmic range of at most $\log2$, so their bracketing integral is a constant, and the argument above bounds their uniform error by $\eps$ except with probability $\delta_j$ for a suitable $C$. The blocks share paths, but a union bound needs no independence, and $\sum_j\delta_j=\delta(3/4-2^{-D-2})<3\delta/4$. Because the $\delta_j$ increase, the $n_j$ decrease, and the number of paths is $n_0\le1+C\eps^{-2}\log\{8(D+1)/\delta\}$. With $d_0=1$ and $d_j=b_j-b_{j-1}$, the number of answers is $N_0=\sum_jd_jn_j$. Since $\log(2/\delta_j)\le\log(8/\delta)+(D-j)\log2$ and $\sum_jd_j(D-j)=K-1$ when $K=2^D$ (and less than $2K$ otherwise, the last block having coefficient zero),
\[
N_0\le K+C\eps^{-2}\{K\log(8/\delta)+2K\log2\}\le K+CK\eps^{-2}\log(32/\delta),
\]
with no hidden $\log\log K$. Path lengths are fixed in advance, so record indicators stay independent; the expected number of queries is $\mu=n_0+\sum_{j\ge1}n_j(H_{b_j}-H_{b_{j-1}})\le n_0H_K$. The audit refuses any query beyond $m_0=\lceil\mu+\sqrt{2\mu\log(4/\delta)}+\frac23\log(4/\delta)\rceil$, which Bernstein's inequality exceeds with probability at most $\delta/4$; without overflow the capped and uncapped audits agree. The total failure probability is below $\delta$, and the three caps are $n_0=O\{\eps^{-2}\log(J/\delta)\}$, $m_0=O(n_0H_K)=O\{J\eps^{-2}\log(J/\delta)\}$ and $N_0=O\{K\eps^{-2}\log(1/\delta)\}$. At $K=1$ everything reduces to estimating a binomial mean.

\section{Proof of Theorem~\ref{thm:frontier}}\label{app:frontier}

\subsection{The upper bound}\label{app:upper}
Let $K\ge2$, $J=\lceil\log_2K\rceil$, $a_\ell=2^{\ell-1}$ and $b_\ell=\min(2^\ell,K)$ for $\ell=1,\ldots,J$. The audit spends $\delta/2$ on an estimate of $\theta_1$ to half-width $\eps/2$ and $\delta/2$ on corrections for $\theta_k-\theta_1$ to total half-width $\eps/2$.

\paragraph{The first budget.} Two valid procedures run side by side, each with half of the error budget, and the audit stops with whichever reaches half-width $\eps/2$ first. The first visits the questions in independent random orders, one answer per question, and uses a fixed-look Chernoff--Kullback--Leibler interval. At any fixed number $n=aM+b$ of answers, the moment generating function of the success count is at most that of $\Bin(n,\theta_1)$: for the $a$ complete orders by the arithmetic--geometric mean inequality applied to $\prod_x\{1+p_1(x)(e^\lambda-1)\}$, and for the $b$ answers of an incomplete random order by Maclaurin's inequality for elementary symmetric polynomials. Chernoff intervals are therefore valid before the first complete round, and a fixed Hoeffding horizon caps this procedure at $O(\eps^{-2}\log(1/\delta))$ answers. Domination of moment generating functions alone does not give exact binomial tails, which is why this procedure uses Chernoff intervals. The second procedure runs the paired sequence of Section~\ref{sec:audit} on complete pairs of rounds with a geometric grid of bets. With $v_1=\vbar_1$ and $Q$ logarithmic in $1/(\eps\delta)$, a grid bet between $\eps/\{8(v_1+\eps)\}$ and $\eps/\{4(v_1+\eps)\}$ reaches half-width $\eps/2$ within $2M+Q(\eps^{-1}+v_1\eps^{-2})$ rows on a high-probability event: over $n$ rows the disagreement count $D$ has mean $nv_1$ and $\Prb\{D>2nv_1+2u\}\le e^{-u}$, and substituting into the radius with $u=\log(1/\eta)$ gives the count. The deterministic cap bounds the contribution of the complementary event by $\eta\,O(\eps^{-2}\log(1/\delta))$; take $\eta=\eps$. The cost of the first budget is at most twice the smaller of the two, $\widetilde O(M_\eps+1/\eps+\vbar_1/\eps^2)$ answers and correctness queries, without an additive $M$ when $M\gg\eps^{-2}$.

\paragraph{Corrections.} An observation at level $\ell$ draws a question uniformly at random and a fresh path of length $b_\ell$, and records, for every budget $k$,
\[
D_{\ell,k}=\begin{cases}0,&k\le a_\ell,\\ W_k-W_{a_\ell},&a_\ell<k\le b_\ell,\\ W_{b_\ell}-W_{a_\ell},&k>b_\ell.\end{cases}
\]
For $k\in(a_L,b_L]$ the levels $\ell<L$ contribute $W_{2^\ell}-W_{2^{\ell-1}}$ and level $L$ contributes $W_k-W_{a_L}$, so $\sum_\ell\E D_{\ell,k}=\theta_k-\theta_1$. Since $b_\ell\le2a_\ell$, Lemma~\ref{lem:coupling} averaged over questions gives $\E D_{\ell,k}^2\le\vbar_{a_\ell}$ for every $k$. Each level runs its own confidence sequence for all of its coordinates and stops when every coordinate has half-width at most $e=\eps/(2J)$; a union over levels, coordinates and sides spends $\delta/2$. Adding the level intervals to the first-budget interval gives half-width at most $\eps$ at every budget, powers of two or not.

One explicit sequence uses the bets $\lambda_h=2^{-h-2}$ for $h=0,\ldots,H$, $H=\lceil\log_2\{(1+e)/e\}\rceil$. For $Z\in[-1,1]$ and $0\le\lambda<1$, $\E e^{\lambda(Z-\E Z)-\psi(\lambda)Z^2}\le1$, because $e^{\lambda z-\psi(\lambda)z^2}\le1+\lambda z$ on $[-1,1]$ and $e^{-\lambda\E Z}(1+\lambda\E Z)\le1$. Ville's inequality then gives, for each bet and each side, an interval valid at all times, and a union over the grid adds $\log(H+1)$ to the logarithmic factor $L$. The radius after $T$ observations with bet $\lambda$ is $\{\psi(\lambda)\sum_{i\le T}Z_i^2+L\}/(\lambda T)$. Write $m_2=\E Z^2\le\vbar_{a_\ell}$. Every interval $[x,2x]$ with $e/\{8(1+e)\}\le x\le1/8$ contains a grid bet, so the grid holds a $\lambda\in[e/\{8(m_2+e)\},e/\{4(m_2+e)\}]$; also $\psi(\lambda)\le\lambda^2$ for $\lambda\le1/2$. On the event $\sum_{i\le T}Z_i^2\le2Tm_2+2L$, which Bernstein's inequality gives with probability at least $1-e^{-L}$ at any fixed $T$, the radius is at most
\[
2\lambda m_2+\frac{2\lambda^2L+L}{\lambda T}\le\frac e2+\frac{L}{2T}+\frac{8L(m_2+e)}{eT},
\]
which is at most $e$ once $T\ge34\,L(m_2/e^2+1/e)$. Outside that event the smallest bet gives a deterministic cap of $O(L/e^2)$ observations, and choosing the failure probability of order $e$ makes its contribution $O(L/e)$. Hence
\[
\E T_\ell=O\bigl\{L\bigl(J^2\vbar_{a_\ell}/\eps^2+J/\eps\bigr)\bigr\}.
\]
An observation at level $\ell$ costs $b_\ell$ answers and at most $H_{b_\ell}$ record queries in expectation.

\paragraph{Summing the levels.} For $\ell\ge3$ the $a_\ell/2$ indices $j\in(a_\ell/2,a_\ell]$ satisfy $\max_{k\ge j}\vbar_k\ge\vbar_{a_\ell}$; for $\ell=1,2$ the single index $j=a_\ell$ does. These index blocks are disjoint, so $\sum_\ell\max(1,a_\ell/2)\vbar_{a_\ell}\le\sumw$, and $b_\ell\le4\max(1,a_\ell/2)$ gives
\[
\sum_\ell b_\ell\vbar_{a_\ell}\le4\sumw,\qquad \sum_\ell b_\ell<3K .
\]
The corrections therefore use $\widetilde O(K/\eps+s/\eps^2)$ answers. For queries, each $\vbar_{a_\ell}$ is at most $a_*=\min(v,s)$, and the same blocks give $\vbar_{a_\ell}\le4s/2^\ell$. Summing the smaller bound, at most $2+\log_2(s/a_*)$ levels contribute $a_*$ each and the rest form a geometric series below $2a_*$, so
\[
\sum_\ell\vbar_{a_\ell}\le a_*\{4+\log_2(s/a_*)\}.
\]
The corrections thus use at most a polylogarithmic factor times $H_K/\eps+a_*\{1+\log(s/a_*)\}/\eps^2$ queries, and the first budget adds $\widetilde O(M_\eps+1/\eps+a_*/\eps^2)$. The audit is never told $v$ or $s$, and it attains both upper bounds at once.

\subsection{Lower bounds on generated answers}\label{app:candlower}
The alternatives below may lie outside $\mathcal C(M,K,v,s)$, which is allowed because a valid audit must cover every law. Each bound is proved by testing, as described in Appendix~\ref{app:proofs}.

\emph{Calibration, at every law.} Let $P$ be any law, $a=\max(\theta_K,1-\theta_K)\ge1/2$, and $y=0$ if $\theta_K\ge1/2$ and $y=1$ otherwise. If the scores of every $P_x$ are bounded above, the alternative $Q_x=(1-\beta)P_x+\beta\,\delta_{(s_x^+,y)}$ adds an answer with correctness $y$ and a score $s_x^+$ above the support of $P_x$. Then $p_K^Q(x)=(1-\beta)^Kp_K(x)+\{1-(1-\beta)^K\}y$, so the target moves by $\{1-(1-\beta)^K\}a$, which exceeds $2\eps$ once $K\{-\log(1-\beta)\}>-\log(1-2\eps/a)$. Since $dP_x/dQ_x\le1/(1-\beta)$, one generated answer has divergence at most $-\log(1-\beta)$, and questions and the audit's choices are otherwise unchanged, so testing gives $\E_PN\ge\kappa_\delta/\{-\log(1-\beta)\}$. Letting $\beta$ decrease to the threshold,
\[
\E_PN\ge\frac{\kappa_\delta K}{-\log(1-2\eps/a)}\ge\frac{\kappa_\delta(1-4\eps)K}{4\eps}\ge\frac{7\kappa_\delta K}{32\eps},
\]
by $-\log(1-z)\le z/(1-z)$ and $a\ge1/2$. If the scores are not bounded above, place the new answer at a score $c_x$ that is not an atom of $P_x$ and has $P_x(S>c_x)\le\gamma/K$. The new answer then wins whenever it is among the $K$ draws, except on an event of probability at most $\gamma$, the target moves by at least $\{1-(1-\beta)^K\}a-\gamma$, the divergence is still at most $-\log(1-\beta)$, and letting $\gamma\to0$ gives the same bound. The bound holds at every law, and in particular at every law of $\mathcal C(M,K,v,s)$.

\emph{Within-question noise.} Take identical questions and two separated score bands. An answer lands in the upper band with probability $q$ and is then correct; otherwise it is incorrect. Choose $q$ so that $p_K=t\le1/2$ and $Kt(1-t)=s$. Then $\vbar_k=p_k(1-p_k)$ increases in $k$, so $\sumw=s$ and $\max_k\vbar_k=s/K\le v$. If $s\ge8K\eps$, then $t\ge8\eps$; lowering $q$ so that $p_K$ falls by $3\eps$ keeps $q$ within a constant factor, and the band indicator of one answer has divergence $O\{\eps^2/(Kt)\}$. Testing gives $\E N=\Omega_\delta(Kt/\eps^2)=\Omega_\delta(s/\eps^2)$. If $s<8K\eps$, then $s/\eps^2<8K/\eps$ and the calibration bound already covers it.

\emph{Telling questions apart.} Give each question a deterministic correctness bit, with continuous scores independent of the bits. Let $\eps\le1/128$ and $21\le M\le\pi/(256\eps^2)$, and draw the bits independently and uniformly. If at least $M/2$ bits are unrevealed, the conditional probability that any interval of width $2\eps$ contains $\theta_1$ is at most
\[
(2\eps M+1)\max_j\Prb\{\Bin(u,1/2)=j\}\le4\eps\sqrt{M/\pi}+2/\sqrt{\pi M}\le\tfrac12 .
\]
Averaging the coverage requirement over the prior, a valid audit reveals more than half the bits with probability at least $1-2\delta\ge1/2$, so $\E N\ge M/4$. For $M<21$ we have $2\eps M<1$, a single unrevealed bit keeps coverage at most $1/2$, and $\E N\ge M/2$. For $M>m=\lfloor\pi/(1024\eps^2)\rfloor$, split $rm$ of the questions into $m$ known groups of $r=\lfloor M/m\rfloor$, give each group one unknown bit, and make the remaining questions incorrect. The groups carry mass $w=rm/M\ge1/2$, so a width-$2\eps$ interval for $\theta_1$ gives a width-$4\eps$ interval for the average of the $m$ bits, and $m\le\pi/\{256(2\eps)^2\}$; the same argument at precision $2\eps$ gives $\E N\ge m/4=\Omega(\eps^{-2})$. One generated answer reveals at most one group's bit, and scores reveal none. When $1/128<\eps\le1/32$, $M_\eps\le128/\eps$ and the calibration bound covers this term. Every revealed bit also needs a correctness query.

The largest of the three bounds is at least a third of their sum, which proves the lower half of \eqref{eq:frontN}.

\subsection{Lower bounds on correctness queries}\label{app:label}
\emph{Every continuous-score law.} Let $P$ be any law with continuous conditional score distributions $F_x$. For $j=0,\ldots,\lfloor\log_{16}K\rfloor$ put $k_j=16^j$ and
\[
B_j(x)=\bigl\{s:(4k_j)^{-1}\le-\log F_x(s)<4/k_j\bigr\}.
\]
These bands are disjoint for each $x$, and the winner at budget $k_j$ falls in $B_j(x)$ with probability $c_0=e^{-1/4}-e^{-4}>3/4$. Averaged over questions, the winner mass of $B_j$ splits into a part carrying label $1$ and a part carrying label $0$, and one of them, say the part $q_y$ with label $y$, is at least $c_0/2$. The alternative $Q_j$ changes only labels inside the bands $B_j(x)$: independently, each answer there with label $y$ has its label replaced by $1-y$ with probability $\alpha_j=3\eps/q_y$. The winner at budget $k_j$ then changes label with probability exactly $\alpha_jq_y=3\eps$, so $\theta_{k_j}$ moves by $3\eps$, and $\alpha_j\le6\eps/c_0<8\eps$. Question and score observations are unchanged, a queried label in $B_j$ has divergence at most $-\log(1-\alpha_j)\le32\eps/3$, and every other label has divergence zero. If $m_j$ counts queries in $B_j$, testing gives
\[
\E_Pm_j\ge\frac{3\,\kl(1-\delta,\delta)}{32\,\eps},
\]
and summing over the disjoint bands gives $\E_Pm=\Omega_\delta(H_K/\eps)$ at the same, arbitrary, law $P$. The bound counts queries to answer occurrences; if a known deterministic grader lets one grade serve every repeat of an answer, distinct grader calls are a different resource.

\emph{Telling questions apart.} The deterministic bits above also force $\Omega_\delta(M_\eps)$ queries.

\emph{The multiscale term.} Let $a_*>0$, $r=a_*/4$, $F=s/(s+r)$ and $h=-\log F$. At every question let the score percentile be uniform, with correctness $\Ber(1-r)$ below $F$ and certain above it. Then $p_k=1-rF^k$, $\vbar_k=rF^k(1-rF^k)\le v$, and $\sumw\le\sum_{k\ge1}rF^k=rF/(1-F)=s$. Let $J_*$ be the number of integers $j\ge0$ with $k_j=16^j\le\min\{K,1/(4h)\}$. Since $a_*/(5s)\le h\le a_*/(4s)$ and $s/a_*\le K$, $J_*$ lies between constant multiples of $1+\log(s/a_*)$. For each such $j$ the band $\{u:(4k_j)^{-1}\le-\log u<4/k_j\}$ lies below $F$, the bands are disjoint, and each holds the winner at budget $k_j$ with probability $c_0$. The alternative lowers the success probability inside band $j$ from $1-r$ to $1-r-\Delta$ with $\Delta=3\eps/c_0<4\eps$, which moves the target by $3\eps$. The alternative success probability is at least $13/16$ and its failure probability at least $r$, so one queried label in the band has divergence at most $256\eps^2/(13r)$, and testing gives $\E_Pm_j\ge13\,\kl(1-\delta,\delta)\,r/(256\eps^2)$. Summing over the $J_*$ bands,
\[
\E_Pm\ge\frac{13\,\kl(1-\delta,\delta)}{1024}\,\frac{a_*J_*}{\eps^2},
\]
which is the logarithmic structure in \eqref{eq:frontm}. This lower bound does not identify the logarithmic factors that the confidence union adds to the upper bound, which is why Theorem~\ref{thm:frontier}(iii) holds up to polylogarithmic factors.

The same law forces generated answers. Let $k_{\rm last}$ be the largest selected $k_j$ and $p_{\rm last}=e^{-1/(4k_{\rm last})}-e^{-4/k_{\rm last}}$ the probability that one answer's percentile lies in its band. Every generated answer lands there with probability $p_{\rm last}$, independently of the past, and a query in the band needs an answer there, so Wald's identity gives $p_{\rm last}\E_PN\ge\E_Pm_{\rm last}\ge13\kappa_\delta r/(256\eps^2)$. Since $p_{\rm last}\le15/(4k_{\rm last})$ and $k_{\rm last}>\min\{K,1/(4h)\}/16\ge s/(16a_*)$,
\[
\E_PN\ge\frac{13\,\kappa_\delta\,s}{61440\,\eps^2}
\]
at the same law $P$. One law in the class thus forces both the label term and the $s/\eps^2$ answer term.

\emph{The first budget alone.} Let $r=\min(2v,s,1/4)$. At every question let the score percentile $U$ be uniform, with correctness 1 when $U\ge1/2$ and $\Ber(1-r)$ otherwise. The budget-$k$ winner lies in the lower half with probability $2^{-k}$, so $1-p_k=2^{-k}r$, $\vbar_k=2^{-k}r(1-2^{-k}r)$ decreases in $k$, $\max_k\vbar_k=\vbar_1\le r/2\le v$ and $\sumw<r\le s$; the law is in $\mathcal C(M,K,v,s)$. The alternative lowers the success probability in the lower half to $1-r-4\eps(1+\gamma)$ for a small $\gamma>0$, which moves $\theta_1$ by $2\eps(1+\gamma)>2\eps$ and leaves scores and upper-half labels unchanged. Its success probability is at least $0.62$ and its failure probability at least $r$, so a lower-half label has divergence at most $16\eps^2(1+\gamma)^2/(0.62\,r)\le26\eps^2(1+\gamma)^2/r$. Testing and $\gamma\to0$ give
\[
\E_Pm\ge\frac{\kappa_\delta\,r}{26\,\eps^2}\ge\frac{\kappa_\delta\,a_*}{26\,\eps^2},
\]
since $r\ge\min(v,s)$ when $v\le1/4$. This law moves only $\theta_1$. With label calibration at $K=1$ and the deterministic bits, certifying $\theta_1$ alone already takes $\Omega_\delta(1/\eps+M_\eps+a_*/\eps^2)$ queries over the class, and by (iii) the whole curve costs at most polylogarithmic factors more.

\emph{Noisy labels.} If in addition $\eta\le\Prb(Y=1\mid x,S)\le1-\eta$ almost surely for some $\eta\in(0,1/2]$, and $\eps\le\eta c_0/6$, raise the success probability by $\Delta=3\eps/c_0\le\eta/2$ inside one band $B_j(x)$ at a time. This moves $\theta_{k_j}$ by exactly $3\eps$, keeps the success probability in $[\eta,1-\eta/2]$, and costs at most $\Delta^2/(\eta/4)=36\eps^2/(\eta c_0^2)$ per queried label in that band. Summing over bands, every valid audit has $\E_Pm\ge(1+\lfloor\log_{16}K\rfloor)\kappa_\delta\eta c_0^2/(36\eps^2)$ at every continuous-score law with this margin.

\section{Proof of Theorem~\ref{thm:instance}}\label{app:instance}

Throughout, the list $\{1,\ldots,M\}$ and its equal weights are known, $P_x$ is the law of one answer $Z=(S,Y)$ at question $x$, and a valid audit covers every $\theta_k$ at once with probability at least $1-\delta$ at every law. For $k$ answers $z_1,\ldots,z_k$, $\widetilde W_k(z_1,\ldots,z_k)$ is the average correctness of the answers that attain the highest score. Given the scores, the tied maxima share one conditional success probability, so the first maximizer and a tie-averaged one have the same expected correctness and $\E\{\widetilde W_k(Z_1,\ldots,Z_k)\mid x\}=p_k(x)$. Write
\[
g_{k,x}(z)=k\bigl[\E\{\widetilde W_k(z,Z_2,\ldots,Z_k)\mid x\}-p_k(x)\bigr],\qquad v_k(x)=\E\{g_{k,x}(Z)^2\mid x\},
\]
so that $\E g_{k,x}(Z)=0$ and $|g_{k,x}|\le k$. With $w_x=Ma_x$, the definition \eqref{eq:gamma} of $\Gamma$ reads $\Gamma=\inf_w\allowbreak\max_kM^{-1}\sum_xv_k(x)/w_x$ over $w\ge0$ with $M^{-1}\sum_xw_x=1$, using $0/0=0$. Lower bounds are proved by testing (Appendix~\ref{app:proofs}).

\subsection{Part (i)}
Fix a valid audit with $\bar N=\E_PN<\infty$; otherwise there is nothing to prove. Let $C(x)$ be the expected number of answers at $x$ and $w_x=MC(x)/\bar N$, so that $M^{-1}\sum_xw_x=1$. For each $k$ put $B_k=M^{-1}\sum_xv_k(x)/(1+w_x)$ and let $k$ maximize it, with $B=B_k$. Since $(1+w)/2$ is an admissible allocation, $B\ge\Gamma/2$.

Suppose first that $\Gamma\ge16\eps K$, so $B\ge8\eps K$. Put $h_x=g_{k,x}/(1+w_x)$, $t=4\eps/B$ and $Q_x(dz)=\{1+th_x(z)\}P_x(dz)$. Since $|th_x|\le tk\le1/2$ and $\E h_x=0$, each $Q_x$ is a law. Expanding $\prod_{i\le k}\{1+th_x(Z_i)\}$ over subsets of $\{1,\ldots,k\}$,
\[
p_k^Q(x)-p_k(x)=t\sum_{i\le k}\E\{\widetilde W_kh_x(Z_i)\}+\E(\widetilde W_kG),\qquad G=\sum_{|A|\ge2}t^{|A|}\prod_{i\in A}h_x(Z_i).
\]
By symmetry the linear term is $tk\,\E\{\E(\widetilde W_k\mid Z_1)h_x(Z_1)\}=t\,\E(g_{k,x}h_x)=tv_k(x)/(1+w_x)$. Products over distinct subsets are orthogonal, so with $r_x=\E h_x^2=v_k(x)/(1+w_x)^2$ we get $\E G^2=(1+t^2r_x)^k-1-kt^2r_x$, and since $\E G=0$ and $|\widetilde W_k-1/2|\le1/2$,
\[
|\E(\widetilde W_kG)|\le\tfrac12\bigl(e^{z_x}-1-z_x\bigr)^{1/2}\le\frac{z_xe^{z_x/2}}{2\sqrt2}\le\frac{z_x}2,\qquad z_x=kt^2r_x\le\frac{(tk)^2}4\le\frac1{16},
\]
using $v_k(x)\le k/4$ from part (iii). Averaging over questions and using $r_x\le v_k(x)/(1+w_x)$,
\[
\theta_k(Q)-\theta_k(P)\ge tB-\frac{kt^2}2B=tB\Bigl(1-\frac{kt}2\Bigr)\ge3\eps .
\]
Because $-\log(1+y)\le-y+y^2$ for $y\ge-1/2$, $\KL(P_x\|Q_x)\le t^2r_x$, and since $w_xr_x\le v_k(x)/(1+w_x)$ the chain rule gives a transcript divergence at most $\bar Nt^2M^{-1}\sum_xw_xr_x\le\bar Nt^2B$. Separation forces this to be at least $\kappa_\delta$, so $\bar N\ge\kappa_\delta B/(16\eps^2)\ge\kappa_\delta\Gamma/(32\eps^2)$. As $K/\eps+\Gamma/\eps^2\le17\Gamma/(16\eps^2)$ in this case, the claim follows. If instead $\Gamma<16\eps K$, then $K/\eps+\Gamma/\eps^2<17K/\eps$, and the every-law calibration bound $\bar N\ge7\kappa_\delta K/(32\eps)$ of Appendix~\ref{app:candlower} is larger than $(\kappa_\delta/128)\cdot17K/\eps$.

For the limit, fix $\eta,\zeta>0$, put $B_k=M^{-1}\sum_xv_k(x)/(\eta+w_x)$ and $B=\max_kB_k\ge\Gamma/(1+\eta)$, because $(\eta+w)/(1+\eta)$ is admissible, and use $h_x=g_{k,x}/(\eta+w_x)$ and $t=2(1+\zeta)\eps/B$. Then $|h_x|\le K/\eta$, $r_x\le K/(4\eta^2)$, $M^{-1}\sum_xr_x\le B/\eta$ and $M^{-1}\sum_xw_xr_x\le B$, and $tK/\eta\le2(1+\zeta)(1+\eta)\eps K/(\eta\Gamma)$ tends to zero uniformly over audits (if $\Gamma=0$ there is nothing to prove). The remainder above, averaged over questions, is at most $\{Kt^2B/(2\sqrt2\eta)\}\exp\{K^2t^2/(8\eta^2)\}=o(tB)$, so the target moves by $2(1+\zeta)\eps\{1-o(1)\}>2\eps$ for small $\eps$. With $-\log(1+y)+y\le y^2/\{2(1-|y|)\}$, $\KL(P_x\|Q_x)\le t^2r_x/\{2(1-tK/\eta)\}$, and the chain rule gives
\[
\bar N\ge\frac{2\kappa_\delta(1-tK/\eta)}{t^2B}\ge\frac{\kappa_\delta\Gamma\{1-o(1)\}}{2(1+\eta)(1+\zeta)^2\eps^2}.
\]
Taking the lower limit and then letting $\eta,\zeta\to0$ proves the claim.

\subsection{Part (ii)}
\emph{Blocks.} A block at question $x$ is $b\ge K$ fresh answers with all correctness values queried. For each $k$, $U_k$ is the average of $\widetilde W_k$ over all $k$-subsets of the block. After sorting the block by score it has a closed form: a tied group of $d$ answers with $c$ answers strictly below contributes its average correctness times $\{\binom{c+d}k-\binom ck\}/\binom bk$. Then $\E U_k=p_k(x)$ and $U_k\in[0,1]$. Write $\sigma_{k,x}^2=\Var(U_k)$. With the orthogonal decomposition $\widetilde W_k-p_k(x)=\sum_{\emptyset\ne A}\phi_{|A|}(Z_A)$ of \citet{hoeffding1948} and $\eta_j=\E\phi_j^2$, one has $\Var(\widetilde W_k)=\sum_{j\le k}\binom kj\eta_j\le1/4$ and $\Var(U_k)=\sum_{j\le k}\binom kj^2\eta_j/\binom bj$. The term $j=1$ gives $b\,\sigma_{k,x}^2\ge k^2\eta_1=v_k(x)$, and for $j\ge2$, $b\binom kj/\binom bj=k\prod_{i=1}^{j-1}(k-i)/(b-i)\le k(k-1)/(b-1)$. Hence
\[
v_k(x)\le b\,\sigma_{k,x}^2\le v_k(x)+\frac{k(k-1)}{4(b-1)} .
\]

\emph{Pilot.} Draw $m$ pairs of blocks at every question. Each pair gives $(U_k-U_k')^2/2\in[0,1/2]$ with mean $\sigma_{k,x}^2$. With $r=\{\log(2MK/\alpha)/(8m)\}^{1/2}$, let $u_{k,x}$ be the minimum of $1/4$ and the average of these values plus $r$. By Hoeffding's inequality and a union bound, on an event $E$ of probability at least $1-\alpha$, $\sigma_{k,x}^2\le u_{k,x}\le\sigma_{k,x}^2+2r$ for all $k$ and $x$. Put $U_{k,x}=bu_{k,x}$; on $E$, $v_k(x)\le U_{k,x}\le v_k(x)+\eta_b$ with $\eta_b=K(K-1)/\{4(b-1)\}+2br$.

\emph{Allocation.} For a nonnegative array $V$ let $\Gamma(V)=\inf_a\max_kM^{-2}\sum_xV_{k,x}/a_x$ over the simplex. For $\gamma\in(0,1)$, combining an optimal $a$ for $V$ with weight $\gamma$ and the uniform allocation with weight $1-\gamma$ and using $(\sqrt V+\sqrt\eta)^2/(\gamma a+(1-\gamma)/M)\le V/(\gamma a)+\eta M/(1-\gamma)$ gives $\Gamma(V+\eta)\le\Gamma(V)/\gamma+\eta/(1-\gamma)$, and optimizing over $\gamma$,
\[
\Gamma(V+\eta)\le\bigl\{\Gamma(V)^{1/2}+\eta^{1/2}\bigr\}^2 .
\]
Let $a^*$ be a pilot-measurable allocation with $G=\max_kM^{-2}\sum_xU_{k,x}/a^*_x\le\Gamma(U)+\tau$, and use $a=(1-\rho)a^*+\rho/M$, so $a_x\ge(1-\rho)a^*_x$ and $a_x\ge\rho/M$. On $E$, $G\le\{\Gamma^{1/2}+\eta_b^{1/2}\}^2+\tau$, because $\Gamma(\cdot)$ is monotone. On every pilot outcome $G\le b/4+\tau$, by comparison with the uniform allocation.

\emph{Fresh blocks and the band.} Given the pilot, choose $\bar N$ below and draw $n_x=\lceil\bar Na_x/b\rceil$ fresh blocks at each question. The estimate $\widehat\theta_k=M^{-1}\sum_x\bar U_{k,x}$, with $\bar U_{k,x}$ the average of the $n_x$ fresh blocks, is a sum of independent terms, each within $1/(Mn_x)\le b/(\rho\bar N)$ of its mean, with total variance at most $V_{\rm up}=\max_kM^{-2}\sum_xu_{k,x}/n_x$ on $E$, and $V_{\rm up}\le G/\{(1-\rho)\bar N\}$ on every pilot outcome. Bernstein's inequality and a union over the $2K$ sides give radius $(2V_{\rm up}\ell)^{1/2}+b\ell/(3\rho\bar N)$ with $\ell=\log\{2K/(\delta-\alpha)\}$ and conditional failure probability at most $\delta-\alpha$ on $E$, so the audit is valid at every law. The choice
\[
\bar N=\Bigl\lceil\Bigl\{\frac{A^{1/2}+(A+4\eps B_0)^{1/2}}{2\eps}\Bigr\}^2\Bigr\rceil,\qquad A=\frac{2G\ell}{1-\rho},\quad B_0=\frac{b\ell}{3\rho},
\]
makes the radius at most $\eps$, and the audit generates at most $2mMb+\bar N+Mb$ answers.

\emph{Schedule.} Take $b=\max(2,K,\lceil\eps^{-1/4}\rceil)$, $m=\lceil\eps^{-3/4}\rceil$, $\rho=\min(1/2,\eps^{1/8})$, $\alpha=\min(\delta/2,\eps)$ and $\tau=\eps$; none depends on the law. As $\eps\to0$: $\eta_b\to0$, the pilot $2mMb=O(M/\eps)$ and $Mb$ are $o(\eps^{-2})$, $\eps B_0\to0$ and $\ell\to\log(2K/\delta)$. On $E$, $\eps^2\bar N\le A\{1+o(1)\}+\eps^2$ and $A\le2\{\Gamma+o(1)\}\log(2K/\delta)$. Off $E$, which has probability at most $\eps$, $G\le b/4+\tau$ gives $\eps^2\bar N=O(b)$, a contribution $O(\eps b)=o(1)$. Hence $\limsup\eps^2\E N\le2\Gamma\log(2K/\delta)$. The allocation $a^*$ is the Neyman allocation for the least favorable mixture of budgets \citep{cochran1977}; learning it from a pilot is the adaptive stratified sampling problem of \citet{carpentier2015}, and the audit charges every pilot answer and label.

\subsection{Part (iii)}
In the decomposition above, $\Var(\widetilde W_k\mid x)\ge k\eta_1=v_k(x)/k$, and $\Var(\widetilde W_k\mid x)\le p_k(x)\{1-p_k(x)\}$ because $\widetilde W_k\in[0,1]$ has mean $p_k(x)$. So $v_k(x)\le kp_k(x)\{1-p_k(x)\}\le k/4$. The uniform allocation gives $\Gamma\le\max_kM^{-1}\sum_xv_k(x)\le\max_kk\vbar_k$, and for each $k$, $k\vbar_k=\sum_{j\le k}\vbar_k\le\sum_{j\le k}\max_{i\ge j}\vbar_i\le\sumw$.

\subsection{Computing \texorpdfstring{$\Gamma$}{Gamma}}\label{app:gamma}
For a question whose answers fall in score groups $1,\ldots,G$ in increasing order, with masses $\mu_g$, success fractions $\eta_g$ and cumulative masses $F_g$ ($F_0=0$), an answer in group $g$ with correctness $y$ has
\[
k\,\E\{\widetilde W_k(z,Z_2,\ldots,Z_k)\mid x\}=k\Bigl\{\eta_gF_g^{k-1}+\sum_{h>g}\eta_h\bigl(F_h^{k-1}-F_{h-1}^{k-1}\bigr)\Bigr\}+(y-\eta_g)\frac{F_g^k-F_{g-1}^k}{\mu_g},
\]
so $v_k(x)$ is a finite sum. For every $\lambda$ in the simplex of budgets, Cauchy--Schwarz gives the lower bound $\Gamma\ge D(\lambda)=\{M^{-1}\sum_x(\sum_k\lambda_kv_k(x))^{1/2}\}^2$, attained by the allocation proportional to the square roots when $\lambda$ is optimal, and any allocation gives an upper bound. We maximize $D$ by conditional-gradient steps and report the best lower and upper values. On the law of the first 4,000 reference answers per question of the MMLU-Pro study, at $K=64$, the bracket is $[0.27961,0.27961]$, against $\sumw=2.566$ on the same law; uniform allocation with all-subsets reuse gives $1.126$, and $\max_kk\vbar_k=2.314$. On the 185 held-out pools at $K=64$ the median $\Gamma$ is $0.34$, the median $\sumw$ is $3.88$, and $\sumw/\Gamma$ has median 12.0 and interquartile range 8.4 to 21.9. On these pools the median ratio of $\max_kk\vbar_k$ to the uniform-allocation value is 2.3, and the median ratio of that value to $\Gamma$ is 3.7.

\section{Proof of the cost bound of Theorem~\ref{thm:paired}}\label{app:pairedcost}
We carry out the argument for the cap $\lambda_{\max}=0.95$ used in the stored-pool study and then explain the change for other caps. Fix a budget $k$ and write $z_k=\vbar_k+\eps$, $L=\log(2K/\alpha_S)$ and $t=\log\{4K\max(1,R_F/2)/\eps\}$, where $R_F$ is the final round. The variance estimate before pair $j$ is half the average disagreement rate of the earlier pairs, with one pseudo-pair of variance $1/4$ and a floor of at most $0.01\,\eps$. By Bernstein's inequality and a union bound over pairs and budgets, outside an event of probability $\eps/2$ the unfloored estimate at pair $j$ is within
\[
\sqrt{\frac{\vbar_kt}{(j-1)M}}+\frac{t}{3(j-1)M}+\frac{1}{4(j-1)M}
\]
of $\vbar_k$. With $A=100{,}000$, this deviation is below $0.01\,z_k$ once $At/z_k$ question visits have been made, and the floor does not change that. Call later pairs \emph{good}. On a good pair,
\[
\lambda_{j}\ge0.94\,\eps/z_k,\qquad \psi(\lambda_{j})\,\vbar_k\le0.57\,\eps\lambda_{j}.
\]
The first follows from the bet rule. For the second, if $\vbar_k\ge\eps/10$ the estimate is at least $0.89\,\vbar_k$, the bet satisfies $\lambda/(1-\lambda)\le\eps/\hat v$, and $\psi(\lambda)\le\lambda^2/\{2(1-\lambda)\}$ gives $\psi(\lambda)\vbar_k\le\lambda\eps\vbar_k/(2\hat v)\le0.57\,\eps\lambda$; if $\vbar_k<\eps/10$, use $\psi(\lambda)/\lambda\le\psi(0.95)/0.95<2.154$. A cap on the bet can only help both comparisons.

Let $P_J=\sum_{j\le J}\psi(\lambda_{j})D_j$, $b=\psi(0.95)<2.05$ and $u=1/(4b)$. The disagreement indicators are independent given the past, with means $2p_k(x)\{1-p_k(x)\}$, and $e^q-1\le1.2\,q$ for $0\le q\le1/4$, so
\[
\exp\Bigl\{u\Bigl[P_J-1.2\sum_{j\le J}2M\psi(\lambda_{j})\vbar_k\Bigr]\Bigr\}
\]
is a nonnegative supermartingale. Ville's inequality, with a union over budgets, bounds $P_J$ for all $J$ by its compensator plus $4bt$, outside an event of probability at most $\eps/2$. The first pair adds at most $6M\eps^2$ to the compensator and all other bad pairs at most $2Abt$. With the good-pair bounds,
\[
P_J\le0.684\,\eps\,\mathsf{mass}_J+7.2\,M\eps^2+500{,}000\,t,\qquad \mathsf{mass}_J=2M\sum_{j\le J}\lambda_{j}.
\]
At most $2M+2At/z_k$ rows lie in bad pairs. If $n=2MJ\ge4M+4At/z_k$, then $\mathsf{mass}_J\ge0.47\,n\eps/z_k$, and substituting into the radius $(P_J+L)/\mathsf{mass}_J$ shows it is at most $\eps$ as soon as $n\ge20M+4{,}000{,}000\,z_k(t+L)/\eps^2$. Rounding up to a whole pair adds fewer than $2M$ rows. So, with probability at least $1-\eps$, simultaneously over budgets,
\[
n_k\le22M+4{,}000{,}000\,(t+L)\frac{\vbar_k+\eps}{\eps^2}.
\]
The constant is loose; it is a device of the proof and does not enter the intervals. An answer at position $j$ is generated only while some budget $k\ge j$ is active, so $N\le\sum_{j\le K}\max_{k\ge j}n_k$, and $\sum_j\max_{k\ge j}(\vbar_k+\eps)\le\sumw+K\eps$. On the complementary event the round cap gives $N\le KR_FM$, and inverting Hoeffding's inequality gives $R_FM=O\{M+\eps^{-2}\log(K/\delta)\}$. Taking expectations proves the bound on $\E N$ in \eqref{eq:pairedcost}. Each row has a length fixed before it is generated, and a row of length $\ell$ has $H_\ell\le H_K$ score records in expectation \citep{renyi}; Wald's identity gives $\E m\le H_K\E n$.

\paragraph{Other caps.} For $\lambda_{\max}\le0.95$ every comparison above holds as stated, except that the bet on a good pair is at least $\min\{\lambda_{\max},0.94\,\eps/z_k\}\ge\lambda_{\max}\cdot0.94\,\eps/z_k$, so the constant in the bound on $n_k$ grows by a factor at most $1/\lambda_{\max}$. For $0.95<\lambda_{\max}<1$, replace the threshold $\eps/10$ by a constant multiple of $\eps$ that depends on $\psi(\lambda_{\max})/\lambda_{\max}$, and enlarge $A$ accordingly. In every case the order in \eqref{eq:pairedcost} is unchanged.

\section{Proof of Theorem~\ref{thm:oracle}}\label{app:oracle}

Write $U$ for the percentile of an answer (with an independent uniform tie-breaker), $w_k(u)=ku^{k-1}$ and $f(u)=\Prb(Y=1\mid U=u)$, so that $\theta_k=\int_0^1w_k(u)f(u)\,du$. Let $L=\log K$. All $o(1)$ terms hold along arbitrary joint sequences with $T/L\to\infty$ and $C/K\to\infty$.

\paragraph{Labels: a lower bound for any adaptive design.} Strengthen the oracle further: it may request a fresh label at any percentile it chooses, and answers are free. In the coordinate $t=-\log u$, the winner of $k$ answers has density $ke^{-kt}$. Fix $A\ge2$ and a mesh $h>0$, let $d=\lfloor(L-2\log A)/h\rfloor$ and $t_j=(A/K)e^{jh}$ for $j=0,\ldots,d$, and use the cells $[t_{j-1},t_j)$ plus two outer cells. Put an independent smooth cosine-squared prior of radius $\rho$ around the Bernoulli mean $1/2$ in each cell. With $p_{kj}$ the winner probability of cell $j$ at budget $k$ and $q_j$ the prior-averaged expected label count in cell $j$ divided by $T$, the van Trees inequality \citep{gilllevit}, which holds for adaptive designs because action probabilities carry no parameter score and label score increments are martingale differences, gives
\[
R\ge\frac{1-4\rho^2}{4T}\max_{k\le K}\sum_j\frac{p_{kj}^2}{q_j+b},\qquad
b=\frac{\pi^2(1-4\rho^2)}{4T\rho^2},\qquad\sum_jq_j\le1 .
\]
Truncating a stopped transcript and passing to the limit in $L_2$ extends it to expected caps; it also covers biased estimators. Completing $q$ to a probability vector and normalizing $q_j+b$, with $R_0=d+2$ cells,
\[
R\ge\frac{1-4\rho^2}{4T(1+R_0b)}\inf_{q\in\Delta}\max_k\sum_j\frac{p_{kj}^2}{q_j}.
\]
For weights $\lambda_k\ge0$ summing to one, Cauchy--Schwarz bounds the infimum below by
\[
\Bigl[\sum_j\Bigl(\sum_k\lambda_kp_{kj}^2\Bigr)^{1/2}\Bigr]^2 .
\] Take $\lambda_k=1/(kH_K)$. On each internal cell, direct integration and
\[
\sum_{k=1}^Kke^{-2kt}=\frac{1-e^{-2Kt}\{1+K(1-e^{-2t})\}}{4\sinh^2t}
\]
give at least $c_A(1-e^{-h})/(2\sqrt{H_K})$ inside the sum, where $c_A=e^{-1/A}\sqrt{1-(1+2A)e^{-2A}}$. Therefore
\[
R\ge\frac{1-4\rho^2}{4T(1+R_0b)}\,\frac{d^2(1-e^{-h})^2c_A^2}{4H_K}.
\]
Choose $A=L$, $h=(L/T)^{1/3}$ and $\rho=\min\{1/8,(R_0/T)^{1/4}\}$. Then $R_0/T=O\{(L/T)^{2/3}\}$, so $h$, $\rho$ and $R_0b$ tend to zero, while $dh/L\to1$ and $H_K/L\to1$, and $R\ge(1-o(1))L/(16T)$ along the actual sequence, with no need to fix $K$ first.

\paragraph{Answers: lower bounds.} Give every label for free. In the oracle experiment take $f_a(u)=1/2+au^{K-1}$ under a shrinking smooth prior around $a=0$. The derivative of $\theta_K$ in $a$ is $K/(2K-1)$, and one answer carries Fisher information at most $4/\{(1-4\rho^2)(2K-1)\}$; with $\rho=\min\{1/8,(K/C)^{1/4}\}$ the van Trees inequality gives $R_{\rm oracle}\ge(1-o(1))K/(8C)$. In the unknown experiment, let an answer score uniformly on $[0,1]$ and be incorrect with probability $1-\lambda/K$, or score uniformly on $[2,3]$ and be correct otherwise. An answer then reveals one Bernoulli bit, and $\theta_K(\lambda)=1-(1-\lambda/K)^K$. A smooth prior around $\lambda=1/2$ with radius $\min\{1/4,(K/C)^{1/3}\}$ gives a squared derivative tending to $e^{-1}$ and information $(2+o(1))/K$ per answer, so $R_{\rm unknown}\ge(1-o(1))K/(2eC)$. Acquisition actions add no likelihood term under adaptive stopping, and both bounds also hold under deterministic caps.

\paragraph{The complete-pool estimator.} For $n\ge K$ answers, sort the scores and set
\[
p_{k,n}(r)=\binom{r-1}{k-1}\Big/\binom nk,\qquad U_{n,k}=\sum_{r=1}^np_{k,n}(r)\,Y_{(r)}.
\]
This is the complete U-statistic that averages the selected label over all subsets of size $k$, so it is unbiased for $\theta_k$, and it satisfies the variance bound \eqref{eq:ustatbound} uniformly over score--correctness laws. To see this, let $R$ be the overlap of two random subsets of size $k$, $Z=R/k$, and $\zeta_r$ the covariance of their selected labels when $r$ answers are shared. Conditioning on the shared maximum's rank and label, expanding the conditional winner mean and integrating by parts gives
\[
\zeta_r=\gamma\Bigl[\theta(1-\theta)-(1-\gamma)\int_0^1z^{-1-\gamma}G(z)\{z-G(z)\}\,dz\Bigr],\qquad\gamma=r/k,
\]
with $G(z)=\int_0^zf(v^{1/k})\,dv$, so that $0\le G'\le1$ and $G(1)=\theta$. The conditional mean used is $g_r(u,y)=u^{k-r}y+(k-r)\int_u^1v^{k-r-1}f(v)\,dv$, and the boundary terms vanish because $0\le G(z)\le z$. After complementing labels so that $\theta\le1/2$, admissibility gives $\max\{0,z-(1-\theta)\}\le G(z)\le\min\{z,\theta\}$. The product $G(z)\{z-G(z)\}$ is concave in $G$, so its minimum over this range is at an endpoint; with $s=1-\theta$, for $z>s$ the upper endpoint's value exceeds the lower one's by $\theta(z-\theta)-s(z-s)=(1-2\theta)(1-z)\ge0$, and for $z\le s$ the lower value is zero. The lower endpoint is therefore the minimizer, which is the threshold mechanism, and it gives $\zeta_r\le s^{2-r/k}-s^2$. With $x=-\log s\in[0,\log2]$ the overlap decomposition becomes
\[
\Var(U_{n,k})\le e^{-2x}\,\E(e^{xZ}-1)\le\frac1{2e}\E Z+\E Z^2,
\]
and the hypergeometric overlap has $\E Z=k/n$ and $\E Z^2\le(k/n)^2+1/n$, which proves \eqref{eq:ustatbound}. All overlap orders are kept; a first-projection approximation at fixed $k$ would not suffice as $K$ grows.

\paragraph{From percentiles to observed ranks.} Let $b_{k,n}(r)=(r/n)^k-\{(r-1)/n\}^k$ and $D_{K,n}=\{1-(K-1)/(2n)\}^{-1}$. For $k\ge2$ and $r\ge k$, $b_{k,n}(r+1)/b_{k,n}(r)\le\{r/(r-1)\}^{k-1}\le r/(r-k+1)=p_{k,n}(r+1)/p_{k,n}(r)$, so $p_{k,n}(r)/b_{k,n}(r)$ increases over its support (the case $k=1$ is immediate). Its value at $r=n$ is $(k/n)/\{1-(1-1/n)^k\}\le D_{K,n}$ by the quadratic lower bound on $1-(1-1/n)^k$, hence
\begin{equation}\label{eq:rankdom}
p_{k,n}(r)\le D_{K,n}\,b_{k,n}(r),\qquad1\le k\le K.
\end{equation}
For a density $q$ with $0<q\le n/t$, let $Q_r$ be its mass on rank bin $r$ and include each sorted answer independently with probability $a_r=tQ_r\le1$. Query the included labels and return
\[
\widehat\theta_k=\frac12+\sum_rp_{k,n}(r)\frac{I_r}{a_r}\Bigl(Y_{(r)}-\frac12\Bigr).
\]
The expected number of labels is $t$. Given the pool, the estimator has mean $U_{n,k}$ and variance $\frac14\sum_rp_{k,n}(r)^2(1/a_r-1)$. Cauchy--Schwarz within rank bins and \eqref{eq:rankdom} give, for any union $B$ of bins, $\sum_{r\in B}p_{k,n}(r)^2/Q_r\le D_{K,n}^2\int_Bw_k(u)^2/q(u)\,du$. On a tail where $q=n/t$ every bin is included and contributes no sampling variance, so before choosing $q$ the coordinate error satisfies
\[
R_k\le\frac k{2en}+\Bigl(\frac kn\Bigr)^2+\frac1n+\frac{D_{K,n}^2}{4t}\int_{\text{outside the saturated tail}}\frac{w_k(u)^2}{q(u)}\,du .
\]

\paragraph{An allocation and its risk.} Work first where $c=nL/(Kt)$ lies in a compact subset of $[1,4]$. Put $Y=\sqrt L$, $J_0=\lceil Y\rceil$, and a tail of width $h=\lceil nY/K\rceil/n$ aligned with the rank bins. Let $g(u)=\max_{j\le J_0}ju^{j-1}$, $C_0=\int g=O(1+\log J_0)$ and $E=g/C_0$. Use
\[
q(u)=\begin{cases}n/t,&u\ge1-h,\\(1-a)\bigl[(1-d)\{(1-u)\log(1/h)\}^{-1}+dE(u)/Z\bigr],&u<1-h,\end{cases}
\]
with $a=(n/t)h$, $d=8C_0/L$ and $Z=\int_0^{1-h}E$. Both components below the tail integrate to one there, so $q$ is a density. Also $C_0=1+\sum_{j<J_0}j^j/(j+1)^{j+1}=O(\log J_0)$, $Z\ge1-J_0h$, and
\[
\sup_{u<1-h}\frac{q(u)}{n/t}\le\frac{1-a}{n/t}\Bigl\{\frac{1-d}{h\log(1/h)}+\frac{dJ_0}{C_0(1-J_0h)}\Bigr\}\longrightarrow0,
\]
so $q\le n/t$ once $K$ is large. With $A_0=\log(1/h)/\{(1-a)(1-d)\}$, integrating the logarithmic component bounds the second moment below the tail by
\[
A_0\frac k{2(2k-1)}(1-h)^{2k-1}\{1+(2k-1)h\}.
\]
For $k\le J_0$ the envelope gives the sharper $L/\{8(1-a)\}$. For $J_0<k\le\lfloor K/\sqrt Y\rfloor$ the display is at most $(1+o(1))L/4$, and for larger $k$ it is $o(L)$ uniformly because $(2k-1)h\to\infty$. The complete-pool variance is negligible against $L/t$ in the middle range and has leading term at most $K/(2en)$ in the upper range. With \eqref{eq:ustatbound} the three ranges give
\[
\sup_P\max_k\operatorname{MSE}(\widehat\theta_k)\le\max\Bigl\{\frac L{16t},\frac K{2en}\Bigr\}+o(L/t).
\]
In the oracle experiment, include each answer independently with probability $q(U_i)/(n/t)$ and use $\widehat\theta_k=\frac12+\frac1t\sum_{i\le n}A_iw_k(U_i)(Y_i-\frac12)/q(U_i)$. Its expected label count is $t$, its variance is $\frac1{4t}\int w_k^2/q-(\theta_k-1/2)^2/n$, and its second moment on the saturated tail is at most $k^2/\{(n/t)(2k-1)\}$. The same three ranges give
\[
\sup_P\max_k\operatorname{MSE}(\widehat\theta_k^{\rm oracle})\le\max\Bigl\{\frac L{16t},\frac K{8n}\Bigr\}+o(L/t).
\]
Both estimators use ordinary answers, and the rank estimator uses only their observed order. The bounds are uniform for $c$ in the compact range.

For general $C$ and $T$, take $n=\min\{\lfloor C\rfloor,\lfloor4KT/L\rfloor\}$ and $t=\min\{T,nL/K\}$. Then $n/K\to\infty$, $t/L\to\infty$ and $c\in[1,4+o(1)]$; below this range the answer term dominates and above it the label term does, so the upper and lower bounds match along oscillating as well as convergent ratios of resources.

\paragraph{Deterministic label caps.} To impose a cap $B=\lfloor T\rfloor$ on every run, recompute either inclusion design with expected count $(1-\xi)B$, where $\xi=\sqrt{24\log B/B}$, draw the whole inclusion mask before requesting any label, and return $1/2$ if the mask exceeds $B$; otherwise use the clipped estimator. Coupling gives $\operatorname{MSE}_{\rm capped}\le\operatorname{MSE}_{\rm uncapped,clipped}+\Prb(\text{overflow})/4$, and Bernstein's inequality bounds the overflow probability by $B^{-8}$, which is $o(L/T)$. Answer counts are deterministic already. The leading constants are therefore the same under both kinds of cap. The oracle has at least the information of the unknown experiment, so $R_{\rm unknown}\ge R_{\rm oracle}$, with equality to first order exactly in the label-limited region.

\section{Experimental details}\label{app:details}

\subsection{Stored score pools}\label{app:pools}

Every stored pool is a finite list of questions with a finite set of scored and graded answers per question. An audit draws answers uniformly with replacement from a question's set, so the pool defines an exact answer law whose curve we compute in closed form, ties included.
\begin{itemize}[leftmargin=*]
\item \emph{Weaver} \citep{weaver2025}, released under the MIT license according to its dataset cards. Eight answer sets: GPQA, MATH500, MMLU and MMLU-Pro questions answered by Llama-3.1-8B-Instruct and Llama-3.1-70B-Instruct, 100 answers per question, each set scored by 19 to 27 reward models and verifiers. The questions of each benchmark are split into two halves once, by question, and the same split is used for every scorer. The 185 held-out pools use the evaluation half (250 to 360 questions). Seven further pools use Weaver's combined score on the full question sets (500 to 719 questions).
\item \emph{CodeRM} \citep{coderm}: the execution results released with the CodeRM code repository. Twenty-four pools: HumanEval+ and MBPP+ \citep{evalplus} and LiveCodeBench tasks, solutions from Llama-3-8B, Llama-3-70B, GPT-3.5-turbo and GPT-4o-mini, and unit tests generated by CodeRM-8B or Llama-3.1-70B. The score is the number of 100 generated unit tests a solution passes, the correctness bit is the benchmark's hidden tests, and each task has 100 solutions (164 to 378 tasks per pool).
\item \emph{CodeContests} \citep{alphacode}: DeepMind-provided non-code materials are CC BY 4.0; third-party materials may have separate terms. One pool of 104 test tasks with 16 human-written Python solutions each, graded by the dataset, and scored by negative length in bytes, so the shortest program is selected.
\end{itemize}
Only the Weaver evaluation halves are held out. The other 32 pools come from sources seen during development and are reported separately.

\subsection{Protocol}

The paired audit's settings were chosen among 26 configurations of stratified betting audits (3 paired, 23 per-row) on the Weaver development halves, by the lowest mean ratio of generated answers to the per-pool cheapest competitor at $K=256$ and 1024, and then frozen: $\alpha_S=0.95\,\delta$, bet cap $\lambda_{\max}=0.95$, one pseudo-pair, the bet $\eps/(\eps+\hat v)$, and $0.05\,\delta$ for the final exact-binomial look. Evaluation ran five independent audits per pool and horizon at $\eps=1/32$ and $\delta=0.05$. During development, provisional evaluation runs of an earlier configuration were launched and one completed; its output was set aside unread until the final method had been fixed and evaluated. The per-pool minimum over three competing designs is an optimistic denominator. The rank-based audit's calibration size of 1,200 paths was chosen after its outcomes on the Weaver pools were known, the nested design uses the fixed look ratio 1.25, and the exact-binomial designs have no tuning parameter.

\subsection{Competing certified audits}\label{app:comparators}

All competitors meet the same requirement: half-width $\eps$ at every budget with simultaneous coverage $1-\delta$ for every answer law. The four designs in the first three items draw questions uniformly at random.
\begin{itemize}[leftmargin=*]
\item \emph{Exact binomial with completion.} A predetermined number of paths is sized so that an exact-binomial interval has half-width at most $\eps$ at every possible count. For each budget in turn, paths are revealed in order until the set of intervals consistent with every value of the unrevealed paths has width at most $2\eps$.
\item \emph{Nested exact binomial.} Looks grow geometrically by a factor $1.25$, the error is split over budget--look pairs in advance, each new path is generated only up to the largest unresolved budget, and a budget retires when its interval is narrow enough. The last look resolves every budget.
\item \emph{Rank-based stopping.} Complete calibration paths (1,200 of them) learn a score threshold above which a correct incumbent can be frozen; a distribution-free rank bound limits later threshold exceedances, and a widened exact-binomial interval covers the stopped paths. A nested version adds early looks with error spending.
\item \emph{Fixed designs.} Every question receives the same number of full paths, with a Hoeffding or an exact-binomial interval and a union over budgets. These designs are balanced but cannot see the within-question variance.
\item \emph{Record design} \citep{fitas2026}. Independent full paths with labels at score records and simultaneous bands, sized for the same $(\eps,\delta)$.
\end{itemize}
Table~\ref{tab:allcomp} lists every design on the held-out pools. The first three form the prespecified comparator of Table~\ref{tab:pools}; at $K=64$ and 256 the cheapest of all six is the cheapest of those three on every held-out pool.

\begin{table}[ht]
\caption{Median ratio of each certified design's cost to the paired audit's on the 185 held-out pools (answers / correctness queries). The fixed designs and the record design were run at $K=64$ and 256.}\label{tab:allcomp}
\centering\small
\begin{tabular}{lccc}
\toprule
design & $K=64$ & $K=256$ & $K=1024$\\
\midrule
exact binomial with completion & 1.49 / 1.37 & 1.98 / 1.47 & 2.74 / 1.66\\
nested exact binomial & 1.80 / 1.70 & 2.13 / 1.75 & 2.73 / 1.78\\
rank-based stopping & 1.37 / 1.57 & 1.54 / 1.68 & 1.90 / 1.88\\
nested rank-based stopping & 1.36 / 1.63 & 1.55 / 1.73 & 1.87 / 1.94\\
fixed Hoeffding & 2.15 / 1.99 & 2.69 / 2.00 & --\\
fixed exact binomial & 1.65 / 1.53 & 2.12 / 1.60 & --\\
record design \citep{fitas2026} & 2.06 / 2.25 & 2.62 / 2.73 & --\\
\bottomrule
\end{tabular}
\end{table}

\paragraph{Uncertified practice.} A within-question bootstrap band (199 replicates) at the budget of the fixed Hoeffding design missed the exact curve in 58 of 925 held-out runs at $K=64$ (6.3\%), and in 70 of 1,085 runs over all pools. The fixed Hoeffding design missed in none of its 370 held-out runs at the same budget.

\subsection{Ablation}\label{app:ablation}

All arms run on the stored pools with the random-number streams of the main study (the balanced exact-binomial arm draws its extra balanced rounds after them), and the paired and per-row costs of every held-out run are reproduced exactly. \emph{Balanced exact binomial} runs the nested exact-binomial looks on balanced rounds, each look rounded up to a whole round and each tail inverted at the corrected level of Lemma~\ref{lem:edgecp}. \emph{Per-row betting} is a stratified confidence sequence that centers each question's outcome at its running mean with a quarter pseudo-count, estimates the variance from the running squared residuals, bets $\min\{0.6,\eps/\hat\sigma^2\}$ and checks after every round; its settings were frozen during development. A \emph{matched} version uses the paired audit's bet $\eps/(\eps+\hat\sigma^2)$ and cap $0.95$. Table~\ref{tab:ablationfull} gives the medians.

\begin{table}[ht]
\caption{Ablation on the 185 held-out pools: median ratio to the cheapest competing certified audit (five-run means).}\label{tab:ablationfull}
\centering\small
\begin{tabular}{lccc}
\toprule
design & $K=64$ & $K=256$ & $K=1024$\\
\midrule
nested exact binomial, random questions & 1.42 & 1.45 & 1.45\\
nested exact binomial, balanced rounds & 1.45 & 1.45 & 1.47\\
per-row betting, balanced rounds & 0.86 & 0.75 & 0.63\\
per-row betting, matched bet and cap & 0.91 & 0.81 & 0.70\\
paired audit & 0.74 & 0.66 & 0.53\\
\midrule
balanced / random nested exact binomial & 1.01 & 1.00 & 1.03\\
paired / better per-row variant & 0.91 & 0.89 & 0.83\\
pools where paired is cheaper than both per-row variants & 185/185 & 185/185 & 185/185\\
\bottomrule
\end{tabular}
\end{table}

\subsection{The all-subsets audit}

The implementation uses bet cap $0.9$, a prior of one pseudo-row for the predictions $c_{x,k}$, a first round with zero bet that only sets the predictions, and the error split $0.95\,\delta$ for the sequence and $0.05\,\delta$ for a fixed-round exact-binomial look on the binary outcome of each full path, inverted at the corrected level of Lemma~\ref{lem:edgecp}. Its unbiasedness requires the answers within a path to be independent given the question.

\subsection{The cost law}\label{app:costlaw}

The law across settings is fitted to the five-run mean costs at $\eps=1/32$ and $K\in\{64,256,1024\}$ and the two-run mean costs at $K=64$ and $\eps\in\{1/16,1/64\}$, with $\sumw$ computed exactly from each pool; its 90th-percentile relative error is 12.0\%. The law used for the MMLU-Pro prediction was fitted earlier, to two paired audits per pool at $K=64$ and $\eps=1/32$ on all 217 pools, run with bet cap $0.6$. It was fitted before any answer of the study, pilot answers included, was generated. The paired audits of the MMLU-Pro study used cap $0.6$ as well. Replaying the three stored audit streams at caps $0.6$ and $0.95$ gives identical best-of-$k$ costs (78,800, 80,400 and 78,200 answers), because the bet $\eps/(\eps+\hat v)$ never exceeded 0.58; for pass@$k$ and majority voting the lower cap binds on 66 to 143 of about 290 bets and moves the cost by a few percent in either direction.

\subsection{The MMLU-Pro study}\label{app:luna}

The protocol was written before any evaluation answer was generated. It fixed a subset of 100 MMLU-Pro questions from mathematics, physics, chemistry and engineering, with 50 further questions for pilots; the model \texttt{gpt-6-luna} (the interface reports no dated version; answers were generated on 27--28 September 2026) at temperature 1, without reasoning effort, with token log-probabilities, at most 400 completion tokens and eight answers per request; a prompt asking for a brief reason followed by a final line ``Answer: X''; the verifier score, which is the mean token log-probability of a response with a valid extracted letter, invalid responses receiving the lowest score; correctness, which is agreement of the extracted letter with the published answer key; three audit streams consumed in generation order with common random numbers across audits; a separate reference stream; and the comparison designs.

Audit streams were generated through the batch interface, and part of the reference stream through the flexible-priority interface with the same parameters. Answers from different generation windows agree: across the 100 questions, single-answer accuracy in the audit streams differs from the reference stream by 0.0022 on average (paired $t=1.30$). Answers within a request behave as independent draws: given the stream and the question, the intraclass correlation of correctness among the eight answers of a request is 0.005. Of the 716,384 answers, 5.2\% had no valid letter and 12 were truncated.

Reference curves come from the 409,448 reference answers: exact finite-sample formulas for best-of-$k$ and pass@$k$, and an exact multinomial computation for majority voting. The all-subsets audits for pass@$k$ and majority voting stopped at the same round in each stream (8, 8 and 7), and their paths were never shortened, so their costs equal the number of rounds times $100\times64$.

\subsection{Computation}\label{app:compute}

Prefix winners along a generated path are updated in time linear in its length. Updating the active budgets after a pair of rounds takes $O(MK)$ operations, so a paired audit of $R$ rounds runs in $O(N+RMK)$ time with $O(MK)$ working memory. The multilevel audit processes an observation of level $\ell$ in $O(b_\ell)$ time with $O(K)$ accumulators. All audits ran on one CPU with eight worker processes; the 3,255 pool--horizon--seed jobs of the stored-pool study, each running all six audits, took 936 seconds in total.

\end{document}